\documentclass[11pt]{article}
\pdfoutput=1

\usepackage{amsmath,amssymb}
\usepackage{bm}
\usepackage{natbib}
\usepackage[usenames]{color}
\usepackage{amsthm}

\usepackage{multirow} 
\usepackage{enumitem}

\usepackage[colorlinks,
linkcolor=red,
anchorcolor=blue,
citecolor=blue
]{hyperref}

\usepackage{setspace}
\usepackage[left=1in, right=1in, top=1in, bottom=1in]{geometry}

\usepackage{xcolor}

\usepackage{mylatexstyle}

\title{\huge Towards Understanding the Spectral Bias of \\Deep Learning}

\author
{
	Yuan Cao\thanks{Equal contribution}~\thanks{Department of Computer Science, University of California, Los Angeles, CA 90095, USA; e-mail: {\tt yuancao@cs.ucla.edu}}
	,
	Zhiying Fang\footnotemark[1]~\thanks{Department of Mathematics, City University of Hong Kong, Kowloon, Hong Kong, China; e-mail: {\tt zyfang4-c@my.cityu.edu.hk}}
	,
	Yue Wu\footnotemark[1]~\thanks{Department of Computer Science, University of California, Los Angeles, CA 90095, USA; e-mail: {\tt ywu@cs.ucla.edu}} 
	,
	Ding-Xuan Zhou\thanks{Department of Mathematics, City University of Hong Kong, Kowloon, Hong Kong, China; e-mail: {\tt mazhou@cityu.edu.hk}} 
	,
	Quanquan Gu\thanks{Department of Computer Science, University of California, Los Angeles, CA 90095, USA; e-mail: {\tt qgu@cs.ucla.edu}}
}
\date{}

\def\Tr{\mathrm{Tr}}
\def\poly{\mathrm{poly}}

\newcommand{\la}{\langle}
\newcommand{\ra}{\rangle}

\begin{document}
\maketitle

\begin{abstract}
An intriguing phenomenon observed during training neural networks is the spectral bias, which states that neural networks are biased towards learning less complex functions. The priority of learning functions with low complexity might be at the core of explaining generalization ability of neural network, and certain efforts have been made to provide theoretical explanation for spectral bias. However, there is still no satisfying theoretical result justifying the underlying mechanism of spectral bias. In this paper, we give a comprehensive and rigorous explanation for spectral bias and relate it with the neural tangent kernel function proposed in recent work. We prove that the training process of neural networks can be decomposed along different directions defined by the eigenfunctions of the neural tangent kernel, where each direction has its own convergence rate and the rate is determined by the corresponding eigenvalue. We then provide a case study when the input data is uniformly distributed over the unit sphere, and show that lower degree spherical harmonics are easier to be learned by over-parameterized neural networks. Finally, we provide numerical experiments to demonstrate the correctness of our theory. Our experimental results also show that our theory can tolerate certain model misspecification in terms of the input data distribution.
\end{abstract}

\section{Introduction}

Over-parameterized neural networks have achieved great success in many applications such as computer vision \citep{he2016deep}, natural language processing \citep{collobert2008unified} and speech recognition \citep{hinton2012deep}.
It has been shown that over-parameterized neural networks can fit complicated target function or even randomly labeled data \citep{zhang2016understanding} and still exhibit good generalization performance when trained with real labels. Intuitively, this is at odds with the traditional notion of generalization ability such as model complexity.
In order to understand neural network training, a line of work \citep{soudry2017implicit,gunasekar2018implicit,gunasekar2018characterizing} has made efforts in the perspective of ``implicit bias'',
which states that training algorithms for deep learning implicitly pose an inductive bias onto the training process and lead to a solution with low complexity measured by certain norms in the parameter space of the neural network.

Among many attempts to establish implicit bias, \citet{rahaman2018spectral} pointed out an intriguing phenomenon called \textit{spectral bias}, which says that during training, neural networks tend to learn the components of lower complexity faster. Similar observation has also been pointed out in \citet{xu2019training,xu2019frequency}.
The concept of spectral bias is appealing because this may intuitively explain why over-parameterized neural networks can achieve a good generalization performance without overfitting. During training, the networks fit the low complexity components first and thus lie in the concept class of low complexity. Arguments like this may lead to rigorous guarantee for generalization.

Great efforts have been made in search of explanations about the spectral bias.
\citet{rahaman2018spectral} evaluated the Fourier spectrum of ReLU networks and empirically showed that the lower frequencies are learned first; also lower frequencies are more robust to random perturbation. \citet{andoni2014learning} showed that for a sufficiently wide two-layer network, gradient descent with respect to the second layer can learn any low degree bounded polynomial. 
\citet{xu2018understanding} provided Fourier analysis to two-layer networks and showed similar empirical results on one-dimensional functions and real data.
\citet{nakkiran2019sgd} used information theoretical approach to show that networks obtained by stochastic gradient descent can be explained by a linear classifier during early training. These studies provide certain explanations about why neural networks exhibit spectral bias in real tasks. But explanations in the theoretical aspect, if any, are to some extent limited. For example, Fourier analysis is usually done in the one-dimensional setting, and thus lacks generality.


Meanwhile, a recent line of work has taken a new approach to analyze neural networks based on the \textit{neural tangent kernel} (NTK) \citep{jacot2018neural}. In particular, they show that under certain over-parameterization condition, 
the neural network trained by gradient descent behaves similarly to the kernel regression predictor using the neural tangent kernel. For training a neural network with hidden layer width $m$ and sample size $n$, recent optimization results on the training loss in the so-called ``neural tangent kernel regime'' can be roughly categorized into the following two families: (i) Without any assumption on the target function (the function used to generate the true labels based on the data input), if the network width $m \geq \poly(n,\lambda_{\min}^{-1})$, where $\lambda_{\min}$ is the smallest eigenvalue of the NTK Gram matrix, then square loss/cross-entropy loss can be optimized to zero \citep{du2018gradient,allen2018convergence,du2018gradientdeep,zou2018stochastic,zou2019improved}; and (ii) If the target function has bounded norm in the NTK-induced reproducing kernel Hilbert space (RKHS), then global convergence can be achieved with milder requirements on $m$. \citep{arora2019fine,arora2019exact,cao2019generalizationsgd,ji2019polylogarithmic}.


Inspired by these works mentioned above in the neural tangent kernel regime, in this paper we study the spectral bias of over-parameterized two-layer neural networks and its connection to the neural tangent kernel. 
Note that a basic connection between neural tangent kernel and the spectral bias can be indicated by the observation that neural networks evolve as linear models \citep{jacot2018neural,lee2019wide} in the NTK regime, which suggests that some corresponding results for linear models \citep{advani2017high} can be applied. However, 
direct combinations of existing techniques cannot show how many hidden nodes can guarantee the learning of simple/complex components, which is the key problem of interest. Therefore, although such combinations of existing results can provide some mathematical intuition, they are not sufficient to explain the spectral bias for neural networks. To give a thorough characterization of spectral bias, we study the training of mildly over-parameterized neural networks. 
We show that, given a training data set that is generated based on a target function, a fairly narrow network, although cannot fit the training data well due to its limited width, can still learn certain low-complexity components of the target function in the eigenspace corresponding to large eigenvalues of neural tangent kernel. As the width of the network increases, more high-frequency components of the target function can be learned with a slower convergence rate. 
As a special case, our result implies that when the  input data follows uniform distribution on the unit sphere, polynomials of lower degrees can be learned by a narrower neural network at a faster rate.
We also conduct experiments to corroborate the theory we establish.

Our contributions are as follows:
\begin{enumerate}[leftmargin = *]
    \item We prove a generic theorem for arbitrary data distributions, which states that under certain sample complexity and over-parameterization conditions, the convergence of the training error along different eigendirections of NTK relies on the corresponding eigenvalues. This theorem gives a more precise control on the regression residual than \citet{su2019learning}, 
    where the authors focused on the case when the labeling function is close to the subspace spanned by the first few eigenfunctions.
    \item We present a characterization of the spectra of the neural tangent kernel that is more general than existing results. 
    In particular, we show that when the input data follow uniform distribution over the unit sphere, the eigenvalues of neural tangent kernel are $\mu_k = \Omega( \max\{k^{-d-1}, d^{-k+1}\} )$, $k\geq 0$, with corresponding eigenfunctions being the $k$-th order spherical harmonics. Our result is better than the bound $\Omega(k^{-d-1})$ derived in \citet{bietti2019inductive} when $d \gg k$, which is in a more practical setting.
    \item We establish a rigorous explanation for the spectral bias based on the aforementioned theoretical results without any specific assumptions on the target function. We show that the error terms from different frequencies are provably controlled by the eigenvalues of the NTK, and the lower-frequency components can be learned with less training examples and narrower networks at a faster convergence rate. 
    
\end{enumerate}


\section{Related Work}
This paper follows the line of research studying the training of over-parameterized neural networks in the neural tangent kernel regime. As mentioned above, \citep{du2018gradient,allen2018convergence,du2018gradientdeep,zou2018stochastic,zou2019improved} proved the global convergence of (stochastic) gradient descent regardless of the target function, at the expense of requiring an extremely wide neural network whose width depends on the smallest eigenvalue of the NTK Gram matrix. Another line of work \citep{arora2019fine,ghorbani2019linearized,arora2019exact,cao2019generalizationsgd,ji2019polylogarithmic} studied the generalization bounds of neural networks trained in the neural tangent kernel regime under various assumptions that essentially require the target function have finite NTK-induced RKHS norm. A side product of these results on generalization is a greatly weakened over-parameterization requirement for global convergence, with the state-of-the-art result requiring a network width only polylogarithmic in the sample size $n$ \citep{ji2019polylogarithmic}. \citet{su2019learning} studied the network training from a  functional approximation perspective, and established a global convergence guarantee when the target function lies in the eigenspace corresponding to the large eigenvalues of the integrating operator $L_{\kappa}f(s) := \int_{\SSS^{d}}\kappa(x,s)f(s) d\tau(s)$, where $\kappa(\cdot,\cdot)$ is the NTK  function  and $\tau(s)$ is the input distribution.

A few theoretical results have been established towards understanding the spectra of neural tangent kernels. To name a few,
\citet{bach2017harmonics} studied two-layer ReLU networks by relating it to kernel methods, and proposed a harmonic decomposition for the functions in the reproducing kernel Hilbert space which we utilize in our proof.
Based on the technique in \citet{bach2017harmonics}, 
\citet{bietti2019inductive} studied the eigenvalue decay of integrating operator defined by the neural tangent kernel on unit sphere by using spherical harmonics. 
\citet{vempala2018gradient} calculated the eigenvalues of neural tangent kernel corresponding to two-layer neural networks with sigmoid activation function.
\citet{basri2019convergence} established similar results as  \citet{bietti2019inductive}, but considered the case of training the first layer parameters of a two-layer networks with bias terms. \citet{yang2019fine} studied the
the eigenvalues of integral operator with respect to the NTK on Boolean cube by Fourier analysis. 
Very recently, \citet{chen2020deep} studied the connection between NTK and Laplacian kernels. 
\citet{bordelon2020spectrum} gave a spectral analysis on the generalization error of NTK-based kernel ridge regression. 
\citet{basri2020frequency} studied the convergence of full training residual with a focus on one-dimensional, non-uniformly distributed data.

A series of papers \citep{gunasekar2017implicit,soudry2017implicit,gunasekar2018characterizing,gunasekar2018implicit,nacson2018stochastic,li2018algorithmic,jacot2020implicit} studied implicit bias problem, aiming to figure out when there are multiple optimal solutions of a training objective function, what kind of nice properties the optimal found by a certain training algorithm would have. Implicit bias results of gradient descent, stochastic gradient descent, or mirror descent for various problem settings including matrix factorization, logistic regression, deep linear networks as well as homogeneous models. The major difference between these results and our work is that implicit bias results usually focus on the parameter space, while we study the functions a neural network prefer to learn in the function space. 



\section{Preliminaries}\label{problemsetup}
In this section we introduce the basic problem setup including the neural network structure and the training algorithm, as well as some background on the neural tangent kernel proposed recently in \citet{jacot2018neural} and the corresponding integral operator. 
\subsection{Notation}
We use lower case, lower case bold face, and upper case bold face letters to denote scalars, vectors and matrices respectively. For a vector $\vb=(v_1,\ldots,v_d)^T\in \RR^d$ and a number $1\leq p < \infty$, we denote its $p-$norm by $\|\vb\|_p = (\sum_{i=1}^d |v_i|^p)^{1/p}$. We also define infinity norm by $\|\vb\|_\infty = \max_i|v_i|$. For a matrix $\Ab = (A_{i,j})_{m\times n} $, we use $\|\Ab \|_{0}$ to denote the number of non-zero entries of $\Ab$, and use $\| \Ab  \|_F = (\sum_{i,j=1}^d A_{i,j}^2)^{1/2}$ to denote its Frobenius norm. Let $\|\Ab \|_p = \max_{\| \vb \|_p \leq 1} \| \Ab \vb\|_p$ for $p\geq 1$, and $\|\Ab \|_{\max} = \max_{i,j} |A_{i,j}|$. For two matrices $\Ab,\Bb \in \RR^{m\times n}$, we define $\la \Ab, \Bb \ra = \Tr(\Ab^\top \Bb)$. We use $\Ab \succeq \Bb$ if $\Ab - \Bb$ is positive semi-definite. For a collection of two matrices $ \Ab = (\Ab_1,\Ab_2 ) \in \RR^{m_1 \times n_1} \otimes \RR^{m_2 \times n_2}$, we denote $\cB(\Ab, \omega ) = \{ \Ab' = (\Ab_1',\Ab_2'): \|\Ab_1' - \Ab_1\|_F, \|\Ab_2' - \Ab_2\|_F \leq \omega \}$. 
In addition, we define the asymptotic notations $\cO(\cdot)$, $\tilde{\cO}(\cdot)$, $\Omega(\cdot)$ and $\tilde\Omega(\cdot)$ as follows. Suppose that $a_n$ and $b_n$ be two sequences. 
We write $a_n = \cO(b_n)$ if $\limsup_{n\rightarrow \infty} |a_n/b_n| < \infty$, and $a_n = \Omega(b_n)$ if $\liminf_{n\rightarrow \infty} |a_n/b_n| > 0$. We use $\tilde{\cO}(\cdot)$ and $\tilde{\Omega}(\cdot)$ to hide the logarithmic factors in $\cO(\cdot)$ and $\Omega(\cdot)$.  

\subsection{Problem Setup}
Here we introduce the basic problem setup. We consider two-layer fully connected neural networks of the form
\begin{align*}
    f_\Wb(\xb) = \sqrt{m}\cdot \Wb_2 \sigma(  \Wb_{1} \xb ),
\end{align*}
where $\Wb_1 \in \RR^{m\times (d+1)}$, $\Wb_2\in\RR^{1\times m}$ are\footnote{Here the input dimension is $d+1$ since in this paper we assume that all training data lie in the $d$-dimensional unit sphere $\SSS^d\in\RR^{d+1}$.} the first and second layer weight matrices respectively, and $\sigma(\cdot) = \max\{0,\cdot\}$ is the entry-wise ReLU activation function. 
The network is trained according to the square loss on $n$ training examples $S = \cbr{(\xb_i, y_i) : i \in [n]}$: 
\begin{align*}
    L_S(\Wb) = \frac{1}{n} \sum_{i=1}^n \rbr{y_i - \theta f_{\Wb}(\xb_i)}^2,
\end{align*}
where $\theta$ is a small coefficient to control the effect of initialization,  
and the data inputs $\left\{\xb_i\right\}_{i=1}^n $ are assumed to follow some unknown distribution $\tau$ on the unit sphere $\SSS^d \in \RR^{d+1}$. Without loss of generality, we also assume that $ |y_i| \leq 1$.

We first randomly initialize the parameters of the network, and run gradient descent for both layers. We present our detailed neural network training algorithm in Algorithm~\ref{alg:GDrandominit}.
\begin{algorithm}[H]
\caption{GD for DNNs starting at Gaussian initialization}\label{alg:GDrandominit}
\begin{algorithmic}
\STATE \textbf{Input:} Number of iterations $T$, step size $\eta$.
\STATE Generate each entry of $\Wb_1^{(0)}$ and $\Wb_2^{(0)}$ from $N(0,2/m)$ and $N(0,1/m)$ respectively.
\FOR{$t=0,1,\ldots, T-1$}
\STATE Update $\Wb^{(t+1)} = \Wb^{(t)} - \eta\cdot \nabla_{\Wb} L_{S}(\Wb^{(t)})$.
\ENDFOR
\STATE \textbf{Output:} $\Wb^{(T)}$.
\end{algorithmic}
\end{algorithm}
The initialization scheme for $\Wb^{(0)}$ given in Algorithm 1 is known as He initialization \citep{he2015delving}. It is consistent with the initialization scheme used in \citet{cao2019generalizationsgd}.

\subsection{Neural Tangent Kernel}
Many attempts have been made to study the convergence of gradient descent assuming the width of the network is extremely large \citep{du2018gradient,li2018learning}. When the width of the network goes to infinity, with certain initialization on the model weights, 
 the limit of inner product of network gradients defines a kernel function, namely the neural tangent kernel \citep{jacot2018neural}. In this paper, we denote the neural tangent kernel as
$$
\kappa(\xb,\xb') = \lim_{m\rightarrow \infty} m^{-1} \la \nabla_{\Wb} f_{\Wb^{(0)}}(\xb), \nabla_{\Wb} f_{\Wb^{(0)}}(\xb') \ra. $$ 
For two-layer networks, standard concentration results gives 
\begin{align}\label{definition:kappa}
    \kappa(\xb,\xb') = \la \xb, \xb'\ra\cdot \kappa_1(\xb,\xb') +   2\cdot \kappa_2(\xb,\xb'),
\end{align}
where
\begin{equation}\label{definition:kernel}
\begin{aligned}
   &\kappa_1(\xb,\xb') = \EE_{\wb\sim N(\mathbf{0}, \Ib)} [\sigma'(\la \wb,\xb \ra)\sigma'(\la \wb,\xb' \ra)],\\
   &\kappa_2(\xb,\xb') = \EE_{\wb\sim N(\mathbf{0}, \Ib)} [\sigma(\la \wb,\xb \ra)\sigma(\la \wb,\xb' \ra)].
\end{aligned}
\end{equation}
Since we apply gradient descent to both layers, the neural tangent kernel is the sum of the two different kernel functions and clearly it can be reduced to one layer training setting. These two kernels are arc-cosine kernels of degree 0 and 1 \citep{cho2009kernel}, which are given as $\kappa_1(\xb,\xb') = \hat \kappa_1(\la\xb,\xb'\ra( \left\|\xb\right\|_2 \left\|\xb'\right\|_2))$, $\kappa_2(\xb,\xb') = \hat \kappa_2(\la\xb,\xb'\ra/( \left\|\xb\right\|_2 \left\|\xb'\right\|_2))$, where 
\begin{equation}
\begin{aligned}
  &\hat\kappa_1(t) = \frac{1}{2\pi} \left( \pi - \arccos{(t)}\right), \\
  &\hat\kappa_2(t) = \frac{1}{2\pi} \left( t \cdot \left( \pi - \arccos{(t)}\right) + \sqrt{1-t^2} \right).
\end{aligned}
\end{equation}

\subsection{Integral Operator}
 The theory of integral operator with respect to kernel function has been well studied in literature \citep{smale2007learning,rosasco2010learning} thus we only give a brief introduction here. Let $L^2_\tau(X)$ be the Hilbert space of square-integrable functions with respect to a  Borel measure $\tau$ from $X \rightarrow \RR$. For any continuous kernel function $\kappa : X \times X \rightarrow \RR$ and $\tau$ we can define an integral operator $L_\kappa$ on $L^2_\tau(X)$ by
\begin{equation}\label{definition:integraloperator}
\begin{aligned}
L_\kappa(f)(\xb)= \int_X \kappa(\xb,\yb) f(\yb) d\tau(\yb),  ~~~ \xb \in X.
\end{aligned}
\end{equation}
It has been pointed out in \cite{cho2009kernel} that arc-cosine kernels are positive semi-definite. Thus the kernel function $\kappa$ defined by (\ref{definition:kappa}) is positive semi-definite being a product and a sum of positive semi-definite kernels. Clearly this kernel is also continuous and symmetric, which implies that the neural tangent kernel $\kappa$ is a Mercer kernel. 

\section{Main Results}\label{mainresults}
In this section we present our main results. We first give a general result on the convergence rate of gradient descent along different eigendirections of neural tangent kernel. Motivated by this result, we give a case study on the spectrum of $L_\kappa$ when the input data are uniformly distributed over the unit sphere $\SSS^d$. At last, we combine the spectrum analysis with the general convergence result to give an explicit convergence rate for uniformly distributed data on the unit sphere.

\subsection{Convergence Analysis of Gradient Descent}\label{sec:convergence}
In this section we study the convergence of Algorithm~\ref{alg:GDrandominit}. Instead of studying the standard convergence of loss function value, we provide a refined analysis on the speed of convergence along different directions defined by the eigenfunctions of $L_\kappa$. We first introduce the following notations.

Let $\{\lambda_i\}_{i\geq 1}$ with $\lambda_1 \geq \lambda_2 \geq \cdots$ be the strictly positive eigenvalues of $L_\kappa$, and $\phi_1(\cdot),\phi_2(\cdot),\ldots$ be the corresponding orthonormal eigenfunctions. Set $\vb_i = n^{-1/2}(\phi_i(\xb_1),\ldots, \phi_i(\xb_n))^\top$, $i = 1,2,\ldots$. Note that $L_\kappa$ may have eigenvalues with multiplicities larger than $1$ and $\lambda_i$, $i\geq 1$ are not distinct. Therefore for any integer $k$, we define $r_k$ as the sum of the multiplicities of the first $k$ distinct eigenvalues of $L_\kappa$. Define $\Vb_{{r_k}} = (\vb_1,\ldots,\vb_{{r_k}})$. 
By definition, $\vb_i$, $i \in [r_k]$ are rescaled restrictions of orthonormal functions in $L_\tau^2(\SSS^{d})$ on the training examples. Therefore we can expect them to form a set of almost orthonormal bases in the vector space $\RR^{n}$. The following lemma follows by standard concentration inequality. The proof is in the appendix. 

\begin{lemma}\label{lemma:projectionconcentration}
Suppose that $ |\phi_i(\xb) | \leq M$ for all $\xb\in \SSS^{d}$ and $i\in[r_k]$. For any $\delta > 0$, with probability at least $ 1 - \delta $, 
\begin{align*}
    \| \Vb_{r_k}^\top \Vb_{r_k} - \Ib \|_{\max} \leq  C M^2 \sqrt{\log(r_k / \delta) / n},
\end{align*}
where $C$ is an absolute constant.
\end{lemma}

Denote $\yb = (y_1,\ldots,y_n)^\top$ and $\hat\yb^{(t)} = \theta\cdot (f_{\Wb^{(t)}}(\xb_1),\ldots,f_{\Wb^{(t)}}(\xb_n) )^\top$ for $t=0,\ldots, T$. Then
Lemma~\ref{lemma:projectionconcentration} shows that the convergence rate of $ \| \Vb_{r_k}^\top (\yb - \hat\yb^{(t)} ) \|_2 $ roughly represents the speed gradient descent learns the components of the target function corresponding to the first $r_k$ eigenvalues. The following theorem gives the convergence guarantee of $ \| \Vb_{r_k}^\top (\yb - \hat\yb^{(t)} ) \|_2$.

\begin{theorem}\label{thm:projectionconvergence}
Suppose $ |\phi_j(\xb) | \leq M$ for $j\in [r_k]$ and $\xb\in \SSS^{d}$. For any $\epsilon,\delta >0$ and integer $k$, if $n \geq \tilde\Omega( \epsilon^{-2}\cdot \max\{ ( \lambda_{r_k } - \lambda_{r_k + 1}  )^{-2} , M^4 r_k^2 \}   ) $, $m \geq \tilde\Omega( \poly(T, \lambda_{r_k}^{-1},\epsilon^{-1}) )$, then with probability at least $1 - \delta$, Algorithm~\ref{alg:GDrandominit} with $\eta = \tilde\cO ( m^{-1} \theta^{-2} )$, $\theta = \tilde \cO (\epsilon)$ satisfies 
\begin{align*}
    &n^{-1/2}\cdot \| \Vb_{r_k}^\top (\yb - \hat\yb^{(T)}) \|_2  \leq 2 ( 1  - \lambda_{r_k})^T \cdot n^{-1/2}\cdot \| \Vb_{r_k}^\top \yb \|_2  + \epsilon.
\end{align*}

\end{theorem}


Theorem~\ref{thm:projectionconvergence} shows that the convergence rate of $ \| \Vb_{r_k}^\top (\yb - \hat\yb^{(t)} ) \|_2 $ is determined by the $r_k$-th eigenvalue $\lambda_{r_k}$. This reveals the spectral bias of neural network training under the NTK regime. Specifically, as long as the network is wide enough and the sample size is large enough, gradient descent first learns the target function along the eigendirections of neural tangent kernel with larger eigenvalues, and then learns the rest components corresponding to smaller eigenvalues. Moreover, by showing that learning the components corresponding to larger eigenvalues can be done with smaller sample size and narrower networks, our theory pushes the study of neural networks in the NTK regime towards a more practical setting. 
For these reasons, we believe that Theorem~\ref{thm:projectionconvergence} to certain extent provides an explanation of the empirical observations given in \citet{rahaman2018spectral}, and demonstrates that the difficulty of a function to be learned by neural network can be characterized in the eigenspace of neural tangent kernel: if the target function has a component corresponding to a small eigenvalue of neural tangent kernel, then learning this function takes longer time, and requires more examples and wider networks. 

Note that the results in this paper are all in the ``neural tangent kernel regime'' \citep{jacot2018neural} or the ``lazy training regime''\citep{chizat2018note}. Therefore, our results share certain common limitations of NTK-type results discussed in \citep{chizat2018note,allen2019can}. However, we believe the study of spectral bias in the NTK regime is still an important research direction. It is worth noting that Theorem~\ref{thm:projectionconvergence} is \textit{not} based on the common NTK-type assumption that the target function belongs to the NTK-induced RKHS \citep{arora2019fine,cao2019generalizationsgd,ji2019polylogarithmic}. Instead, Theorem~\ref{thm:projectionconvergence} works for arbitrary labeling of the data, and is therefore rather general. 
We believe the characterization of spectral bias under general settings beyond NTK regime is an important future work direction.

Our work follows the same intuition as recent results studying the residual dynamics of over-parameterized two-layer neural networks \citep{arora2019fine,su2019learning}. 
Compared with \citet{su2019learning}, the major difference is that while \citet{su2019learning} studied the full residual $\| \yb - \hat\yb^{(T)} \|_2$ and required that the target function lies approximately in the eigenspace of large eigenvalues of the neural tangent kernel, our result in Theorem~\ref{thm:projectionconvergence} works for arbitrary target function, and shows that even if the target function has very high frequency components, its components in the eigenspace of large eigenvalues can still be learned very efficiently by neural networks. We note that although the major results in \citet{arora2019fine} are presented in terms of the full residual, certain part of their proof in \citet{arora2019fine} can indicate the convergence of projected residual. However, \citet{arora2019fine} do not build any quantitative connection between the Gram matrix and kernel function. Since the eigenvalues and eigenfunctions of NTK Gram matrix depend on the exact realizations of the $n$ training samples, they are not directly tractable for the study of spectral bias. Moreover, \citet{arora2019fine} focus on the setting where the network is wide enough to guarantee global convergence, while our result works for narrower networks for which global convergence may not even be possible. More recently, \citet{bordelon2020spectrum} studied the solution of NTK-based kernel ridge regression. Compared with thier result, our analysis is directly on the practical neural network training procedure, and specifies the width requirement for a network to successfully learn a certain component of the target function. Another recent work by \citet{basri2020frequency} provided theoretical analysis for one-dimensional non-uniformly distributed data, and only studied the convergence of the full residual vector. In comparison, our results cover high-dimensional data, and provide a more detailed analysis on the convergence of different projections of the residual.

    


\subsection{Spectral Analysis of NTK for Uniform Distribution}\label{sec:spectral}
We now study the case when the data inputs are uniformly distributed over the unit sphere as an example where the eigendecompositon of NTK can be calculated. We present our results (an extension of Proposition~5 in \citet{bietti2019inductive}) of spectral analysis of neural tangent kernel in the form of a Mercer decomposition, which explicitly gives the eigenvalues and eigenfunctions of NTK.
\begin{theorem}\label{theorem:spectralanalysis}
For any $\xb,\xb' \in \SSS^d \subset \RR^{d+1}$, we have the Mercer decomposition of the neural tangent kernel $\kappa : \SSS^d \times \SSS^d \rightarrow \RR$,
\begin{align}\label{ntkmercerdecomposition}
    \kappa\left(\xb,\xb'\right) = \sum_{k=0}^\infty \mu_k \sum_{j=1}^{N(d,k)} Y_{k,j}\left(\xb\right) Y_{k,j}\left(\xb'\right),
\end{align}
where $Y_{k,j}$ for $j=1,\cdots,N(d,k)$ are linearly independent spherical harmonics of degree $k$ in $d+1$ variables with $N(d,k) = \frac{2k+d-1}{k} \tbinom{k+d-2}{d-1} $ and orders of $\mu_k$ are  given by
\begin{align*}
    &\mu_0 = \mu_{1} = \Omega(1), ~ \mu_{k} = 0,k=2j+1,\\
    &\mu_{k} = \Omega (\max\{ d^{d+1} k^{k-1} (k+d)^{-k-d}, d^{d+1} k^k (k+d)^{-k-d-1}, d^{d+2} k^{k-2} (k+d)^{-k-d-1} \}), k = 2j,
\end{align*}
where $j\in \NN^+$.
Specifically, we have $\mu_k = \Omega\left(k^{-d-1}\right)$ when $k \gg d$ and $\mu_k = \Omega\left(d^{-k+1}\right)$ when $d \gg k$, $k = 2,4,6,\ldots$.

\end{theorem}
\begin{remark}
The $\mu_k$'s in Theorem~\ref{theorem:spectralanalysis} are the distinct eigenvalues of the integral operator $L_\kappa$ on ${L}^2_{\tau_d}(\SSS^d)$ defined by 
\begin{equation*}
\begin{aligned}
L_\kappa(f)(\yb)= \int_{\SSS^d} \kappa(\xb,\yb) f(\xb) d\tau_d(\xb),  ~~~ f \in {L}^2_{\tau_d}(\SSS^d),
\end{aligned}
\end{equation*}
where $\tau_d$ is the uniform probability measure on unit sphere $\SSS^d$. Therefore the eigenvalue $\lambda_{r_k}$ in Theorem~\ref{thm:projectionconvergence} is just $\mu_{k-1}$ given in Theorem~\ref{theorem:spectralanalysis} when $\tau_d$ is uniform distribution. 
\end{remark}
\begin{remark}
\citet{vempala2018gradient} studied two-layer neural networks with sigmoid activation function, and proved that in order to achieve $\epsilon_0 + \epsilon$ error, it requires $T=(d+1)^{\cO(k)\log{(\|f^*\|_2 / \epsilon)}}$ iterations and
$m=(d+1)^{\cO(k)\poly{(\|f^*\|_2/\epsilon)}}$ wide neural networks, where $f^*$ is the target function, and $\epsilon_0$ is certain function approximation error.
Another highly related work is \cite{bietti2019inductive}, which gives $\mu_k = \Omega(k^{-d-1})$. The order of eigenvalues we present appears as $\mu_k = \Omega( \max(k^{-d-1}, d^{-k+1}))$. This is better when $d \gg k$, which is closer to the practical setting.
\end{remark}

\subsection{Convergence for Uniformly Distributed Data}\label{sec:convergence_uniformdistribution}
In this subsection, we combine our results in the previous two subsections and give explicit convergence rate for uniformly distributed data on the unit sphere. 
\begin{corollary}\label{col:largek}
Suppose that $k \gg d$, and the sample $\left\{\xb_i\right\}_{i=1}^n $ follows the uniform distribution $\tau_d$ on the unit sphere $\SSS^d $. For any $\epsilon,\delta >0$ and integer $k$, if $n \geq \tilde\Omega( \epsilon^{-2}\cdot \max\{  k^{2d+2}, k^{2d-2} r_k^2 \}   )$, $m \geq \tilde\Omega( \poly(T, k^{d+1},\epsilon^{-1}) )$, then with probability at least $1 - \delta$,  Algorithm~\ref{alg:GDrandominit} with $\eta = \tilde\cO ( (m \theta^2 )^{-1} )$, $\theta = \tilde\cO( \epsilon)$ satisfies 
\begin{align*}
    &n^{-1/2}\cdot \| \Vb_{r_k}^\top (\yb - \hat\yb^{(T)}) \|_2 \leq 2 \left( 1  - \Omega\left(k^{-d-1}\right) \right)^T \cdot n^{-1/2}\cdot \| \Vb_{r_k}^\top \yb \|_2  + \epsilon,
\end{align*}
where $r_k = \sum_{k'=0}^{k-1} N(d,k')$ and $\Vb_{{r_k}} = ( n^{-1/2}\phi_j(\xb_i) )_{n\times r_k}$ with $\phi_1,\ldots,\phi_{r_k}$ being a set of orthonomal spherical harmonics of degrees up to $k-1$.
\end{corollary}

\begin{corollary}\label{col:larged}
Suppose that $d \gg k$, and the sample $\left\{\xb_i\right\}_{i=1}^n $ follows the uniform distribution $\tau_d$ on the unit sphere $\SSS^d $. For any $\epsilon,\delta >0$ and integer $k$, if $n \geq \tilde\Omega( \epsilon^{-2} d^{2k-2} r_k^2    )$, $m \geq \tilde\Omega( \poly(T, d^{k-2},\epsilon^{-1}) )$, then with probability at least $1 - \delta$,  Algorithm~\ref{alg:GDrandominit} with $\eta = \tilde\cO ( (m \theta^2 )^{-1} )$, $\theta = \tilde\cO( \epsilon)$ satisfies 
\begin{align*}
    &n^{-1/2}\cdot \| \Vb_{r_k}^\top (\yb - \hat\yb^{(T)}) \|_2  \leq 2 \left( 1  - \Omega\left(d^{-k+2} \right)\right)^T \cdot n^{-1/2}\cdot \| \Vb_{r_k}^\top \yb \|_2  + \epsilon,
\end{align*}
where $r_k = \sum_{k'=0}^{k-1} N(d,k')$ and $\Vb_{{r_k}} = ( n^{-1/2}\phi_j(\xb_i) )_{n\times r_k}$ with $\phi_1,\ldots,\phi_{r_k}$ being a set of orthonomal spherical harmonics of degrees up to $k-1$.
\end{corollary}
Corollaries~\ref{col:largek} and \ref{col:larged} further illustrate the spectral bias of neural networks by providing exact calculations of $\lambda_{r_k}$, $\Vb_{r_k}$ and $M$ in Theorem~\ref{thm:projectionconvergence}. They show that if the input distribution is uniform over unit sphere, then spherical harmonics with lower degrees are learned first by wide neural networks.

\begin{remark}
    In Corollaries~\ref{col:largek} and \ref{col:larged}, the conditions on $n$ and $m$ depend exponentially on either $k$ or $d$. We would like to emphasize that such exponential dependency is reasonable and unavoidable. 
    Take the $d \gg k$ setting as an example. The exponential dependency in $k$ is a natural consequence of the fact that in high dimensional space, there are a large number of linearly independent low-degree polynomials. When only $n$ data points are used for training, it is only reasonable to expect to learn less than $n$ independent components of the true function. Therefore it is unavoidable to assume 
    \begin{align*}
        n &\geq r_k = \sum_{k'=0}^{k-1} N(d,k') = \sum_{k'=0}^{k-1}\frac{2k'+d-1}{k'} \tbinom{k'+d-2}{d-1} = \sum_{k'=0}^{k-1}\frac{2k'+d-1}{k'} \tbinom{k'+d-2}{k'-1} =\Omega(d^{k-1}).
    \end{align*}
    Similar arguments can apply to the $k \gg d$ setting.
\end{remark}


\section{Experiments}\label{experiments}

In this section we present experimentral results to verify our theory. Across all tasks, we train a two-layer neural networks with 4096 hidden neurons and initialize it exactly as defined in the problem setup. The optimization method is vanilla gradient descent, and the training sample size for all results is $1000$. 

\subsection{Learning Combinations of Spherical Harmonics}
\label{subsec4.1}
We first experimentally verify our theoretical results by learning linear combinations of spherical harmonics with data inputs uniformly distributed over unit sphere. Define
\begin{align*}
    f^*(\xb) = a_1 P_1(\dotp{\bm{\zeta}_1}{\xb}) + a_2  P_2(\dotp{\bm{\zeta}_2}{\xb}) + a_4  P_4(\dotp{\bm{\zeta}_4}{\xb}),
\end{align*}
where the $P_k(t)$ is the Gegenbauer polynomial, and $\bm{\zeta}_k$, $k=1,2,4$ are fixed vectors that are independently generated from uniform distribution on unit sphere in $\RR^{10}$.
Note that according to the addition formula 
$\sum_{j=1}^{N(d,k)} Y_{k,j}(\xb)Y_{k,j}(\yb) = N(d,k) P_k(\dotp{\xb}{\yb})$,
every normalized Gegenbauer polynomial is a spherical harmonic, so $f^*(\xb)$ is a linear combination of spherical harmonics of order 1,2 and 4. The higher odd-order Gegenbauer polynomials are omitted because the spectral analysis showed that $\mu_k = 0$ for $ k = 3, 5, 7 \dots$.

Following our theoretical analysis, 
we denote $\vb_k = n^{-1/2}(P_k(\xb_1),\ldots, P_k(\xb_n))$. By Lemma \ref{thm:projectionconvergence} $\vb_k$'s are almost orthonormal. So we define the (approximate) projection length of residual $\rb^{(t)}$ onto $\vb_k$ at step $t$ as $ \hat{a}_k = |\vb_k^{\top} \rb^{(t)}| $,
where $\rb^{(t)} = (f^*(\xb_1) - \theta f_{\Wb^{(t)}}(\xb_1),\dots,f^*(\xb_n) - \theta f_{\Wb^{(t)}}(\xb_n))$
and $f_{\Wb^{(t)}}(\xb)$ is the neural network function.


\begin{figure}[H]
     \centering
     \subfigure[same scale]{\includegraphics[width=0.47\textwidth]{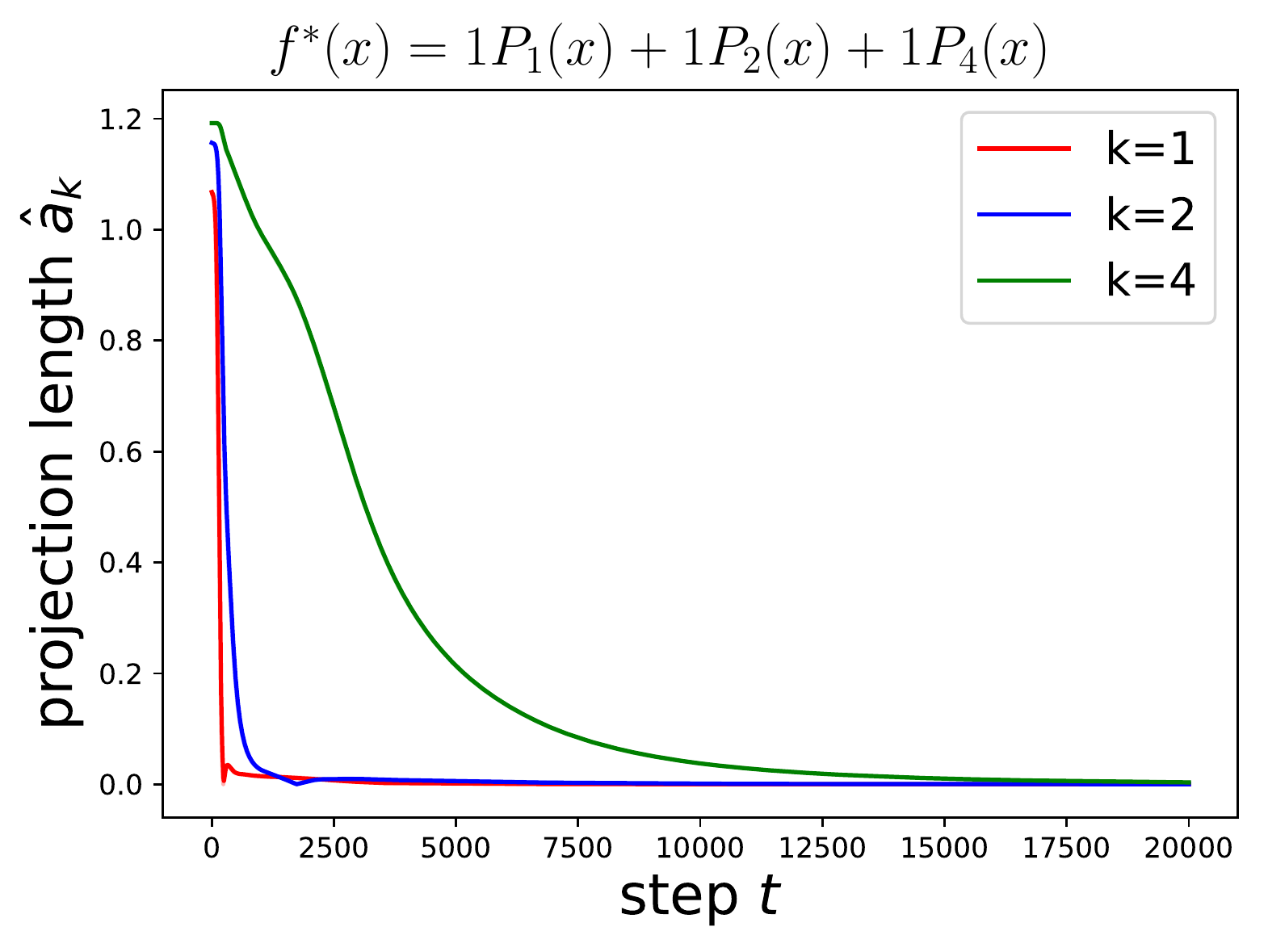}}
     \subfigure[different scale]{\includegraphics[width=0.47\textwidth]{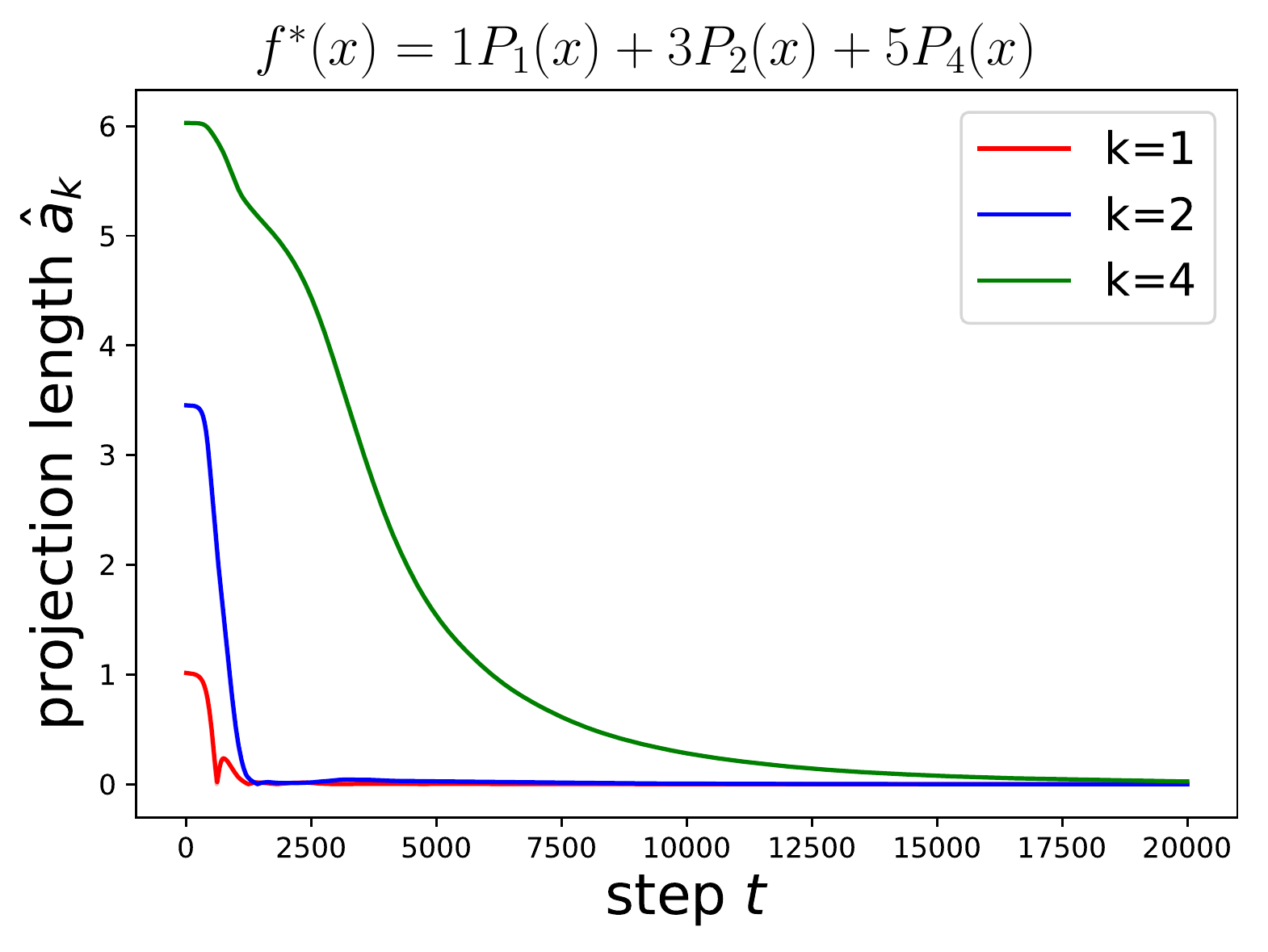}}
    \caption{Convergence curve of projection lengths. (a) shows the curve when the target function have the same scale for different components. (b) shows the curve when the higher-order components have larger scale.
    }
    \label{fig1}
\end{figure}

The results are shown in Figure \ref{fig1}. It can be seen that the residual at the lowest frequency ($k=1$) converges to zero first and then the second lowest ($k=2$). The highest frequency component is the last one to converge.
Following the setting of \citet{rahaman2018spectral} we assign high frequencies a larger scale in Figure \ref{fig1} (b), expecting that larger scale will introduce a better descending speed. Still, the low frequencies are learned first.

Note that $\hat a_k$ is the projection length onto an approximate vector. In the function space, we can also project the residual function $r(\xb) = f^*(\xb) - \theta f_{\Wb^{(t)}}(\xb)$ onto the orthonormal Gegenbauer functions $P_k(\xb)$. Replacing the training data with randomly sampled data points $\xb_i$ can lead to a random estimate of the projection length in function space. Experiments in this setting can be found in the appendix.


We also verify the linear convergence rate proved in Theorem~\ref{thm:projectionconvergence}. Figure \ref{fig4} presents the convergence curve in logarithmic scale. 
We can see from Figure \ref{fig4} that the convergence of different projection length is close to linear convergence, which is well-aligned with our theoretical analysis. Note that we performed a moving average of range $20$ on these curves to avoid the heavy oscillation especially at late stage.

\begin{figure}[htb]
     \centering
     \subfigure[same scale]{\includegraphics[width=0.47\textwidth]{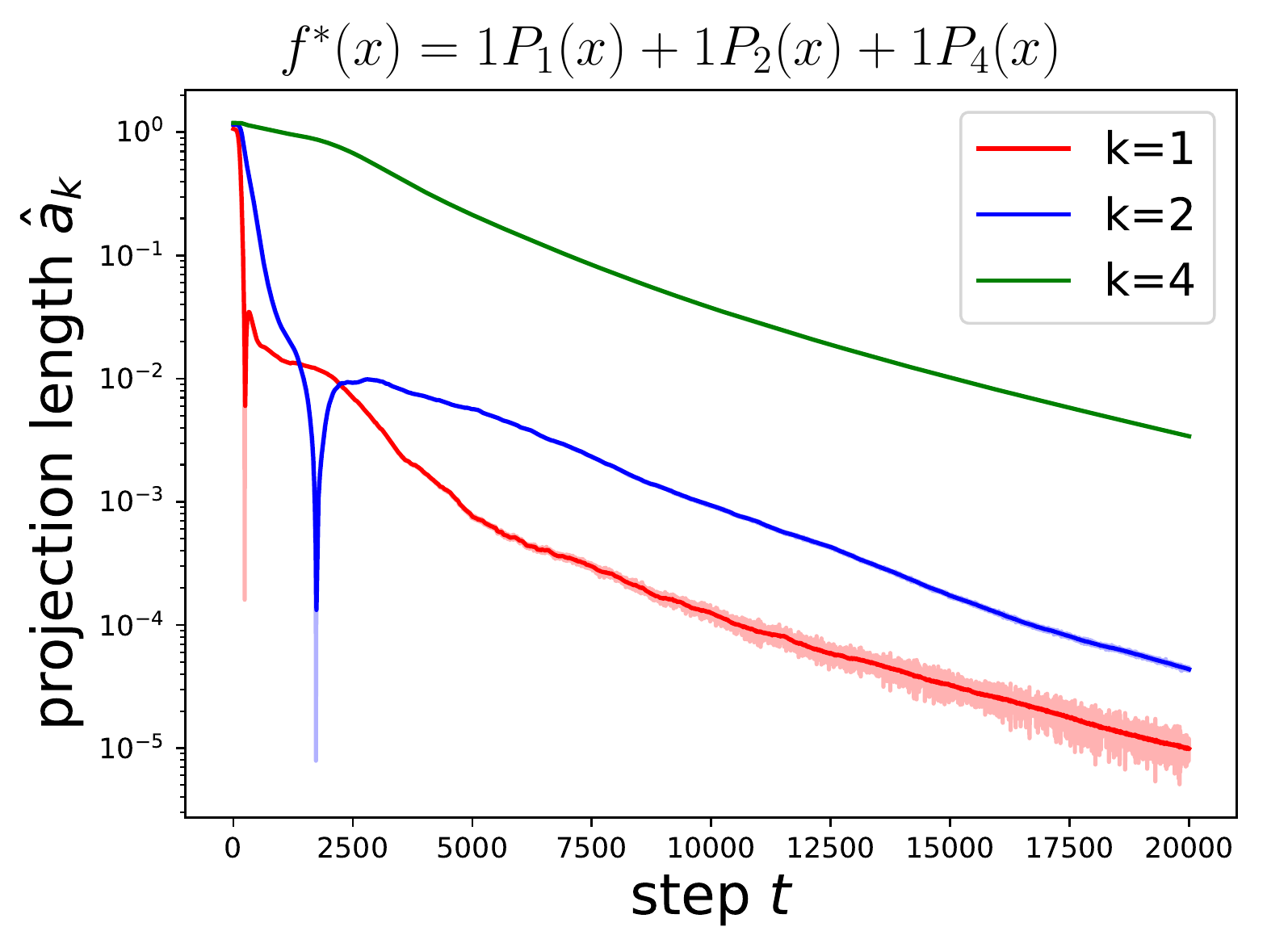}}
     \subfigure[different scale]{\includegraphics[width=0.47\textwidth]{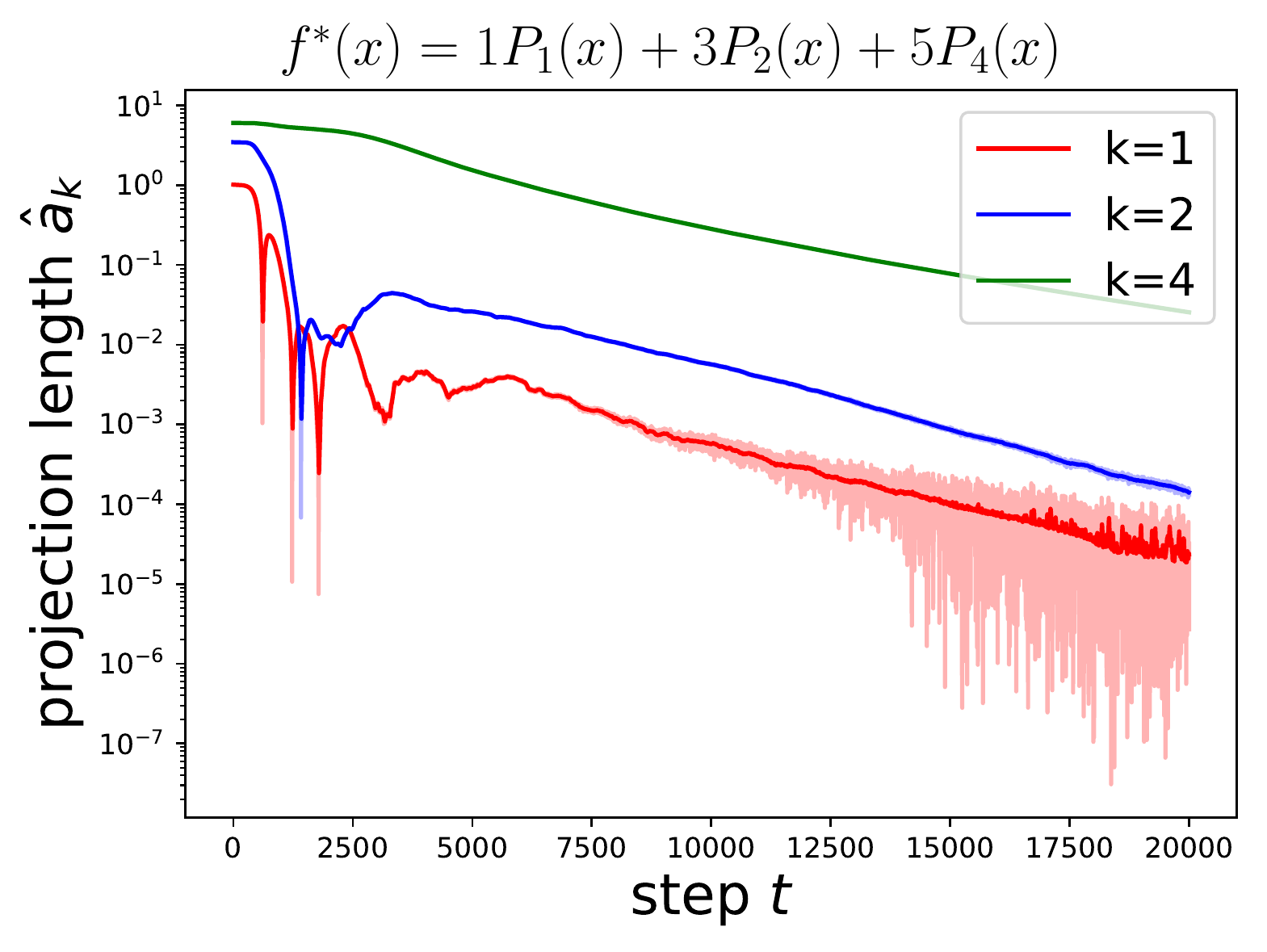}}
    \caption{Log-scale convergence curve for projection lengths. (a) shows the curve when the target function have the same scale for different components. (b) shows the curve when the higher-order components have larger scale.
    }
    \label{fig4}
\end{figure}

\subsection{Learning Functions of Simple Forms}
Apart from the synthesized low frequency function, we also show the dynamics of more general functions' projection to $P_k(x)$. These functions, though in a simple form, have non-zero high-frequency components. In this subsection we show our results still apply when all frequencies exist in the target functions, which
are given by $f^*(\xb)=\sum_{i} \cos(a_i \dotp{\bm{\zeta}}{\xb})$ and $f^*(\xb)=\sum_{i} \dotp{\bm{\zeta}}{\xb}^{p_i}$, where $\bm{\zeta}$ is a fixed unit vector.

Figure \ref{fig2} verifies the result of Theorem~\ref{thm:projectionconvergence} with more general target functions, and backs up our claim that Theorem~\ref{thm:projectionconvergence} does not make any assumptions on the target function. Notice that in the early training stage, not all the curves monotonically descend. We believe this is due to the unseen components' influence on the gradient. Again, as the training proceeds, the residual projections converge at the predicted rates. 

\begin{figure}[H]
     \centering
     \subfigure[cosine function]{\includegraphics[width=0.47\textwidth]{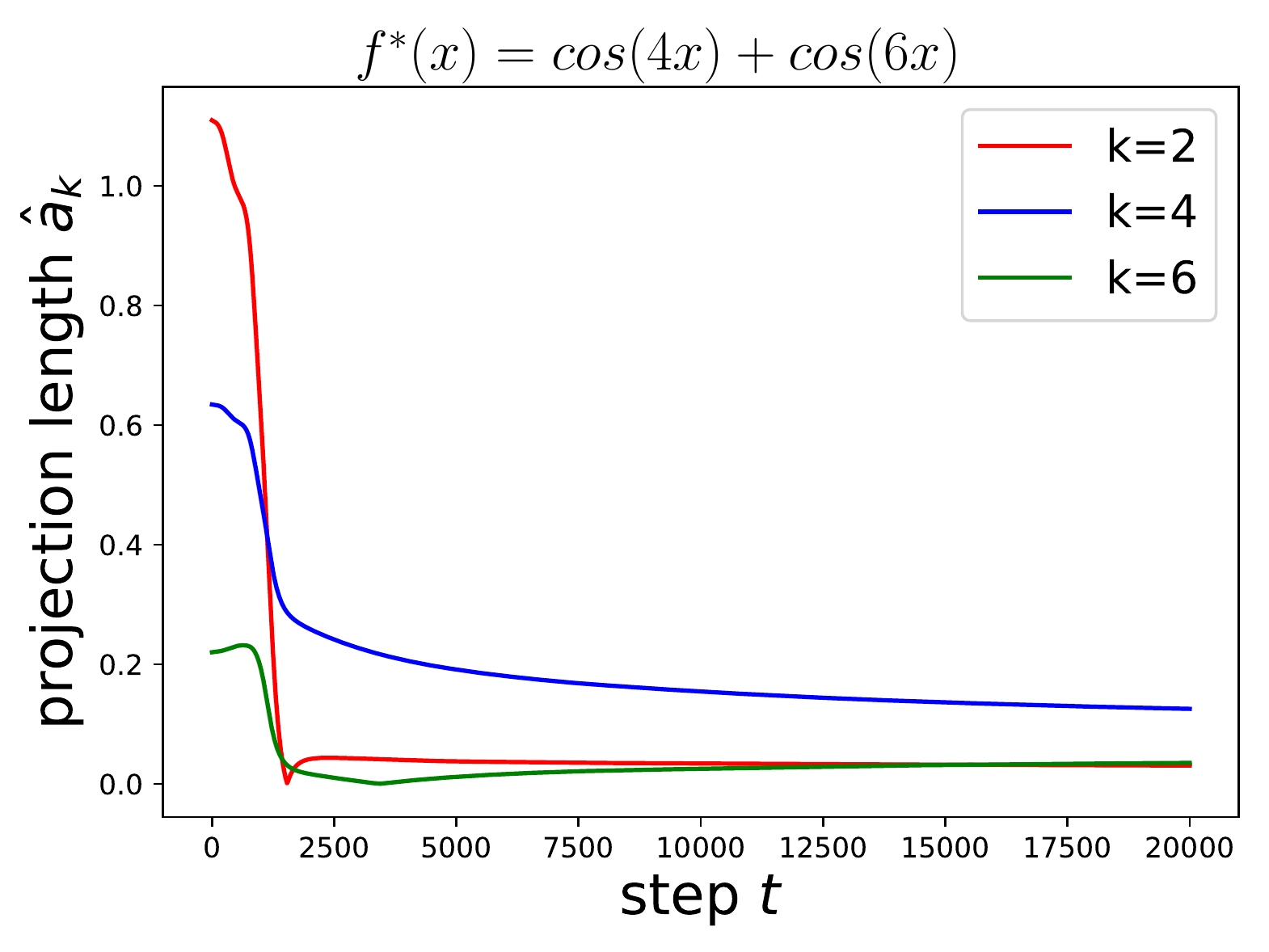}}
     \subfigure[even polynomial]{\includegraphics[width=0.47\textwidth]{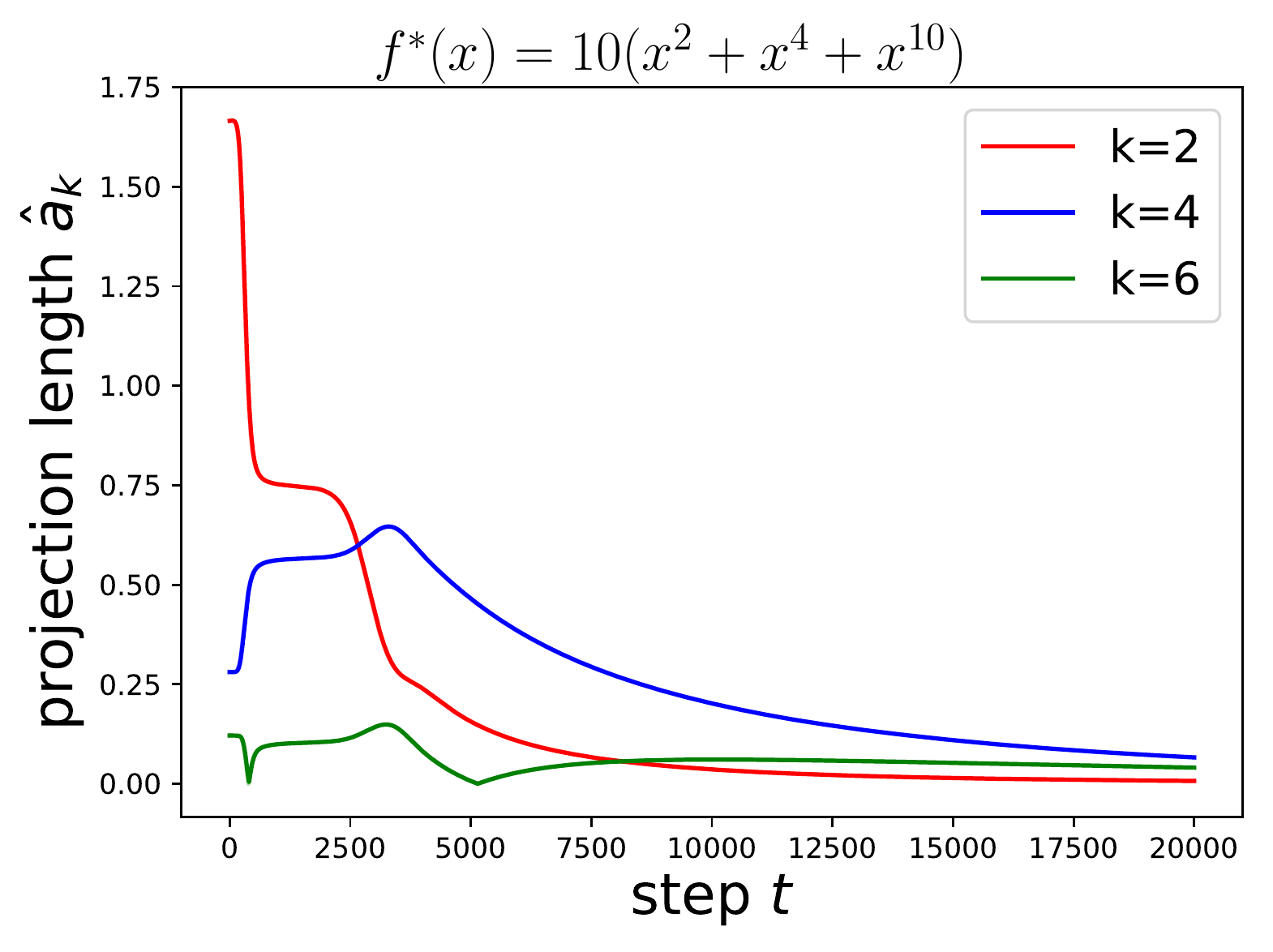}}
    \caption{Convergence curve for different components. (a) shows the curve of a trigonometric function. (b) shows the curve of a polynomial with even degrees.
    }
    \label{fig2}
\end{figure}



\subsection{Non-uniform Input Data  Distributions}

In this subsection, we provide experimental results for non-uniformly distributed input data. Note that the eigendecomposition of NTK for general multi-dimensional input distributions do not necessarily have good analytic forms. Therefore, here we treat the non-uniform distribution of the input data as model misspecification, and test whether residual projections onto spherical harmonics of different degrees can still exhibit various convergence speed. We consider three examples of non-uniform distributions: (i) piece-wise uniform distribution, (ii) normalized non-isotropic Gaussian, and (iii) normalized Gaussian mixture.

\textbf{Piece-wise uniform distribution }
We divide the unit sphere into two semi-spheres along a randomly drawn direction $\bzeta_0$. A data input is then generated as follows: with probability $1/4$, draw the input uniformly over the first unit sphere; with probability $3/4$, draw the input uniformly over the second unit sphere. The results are shown in Figure~\ref{fig1-2}. 

\begin{figure}[H]
     \centering
     \subfigure[same scale]{\includegraphics[width=0.47\textwidth]{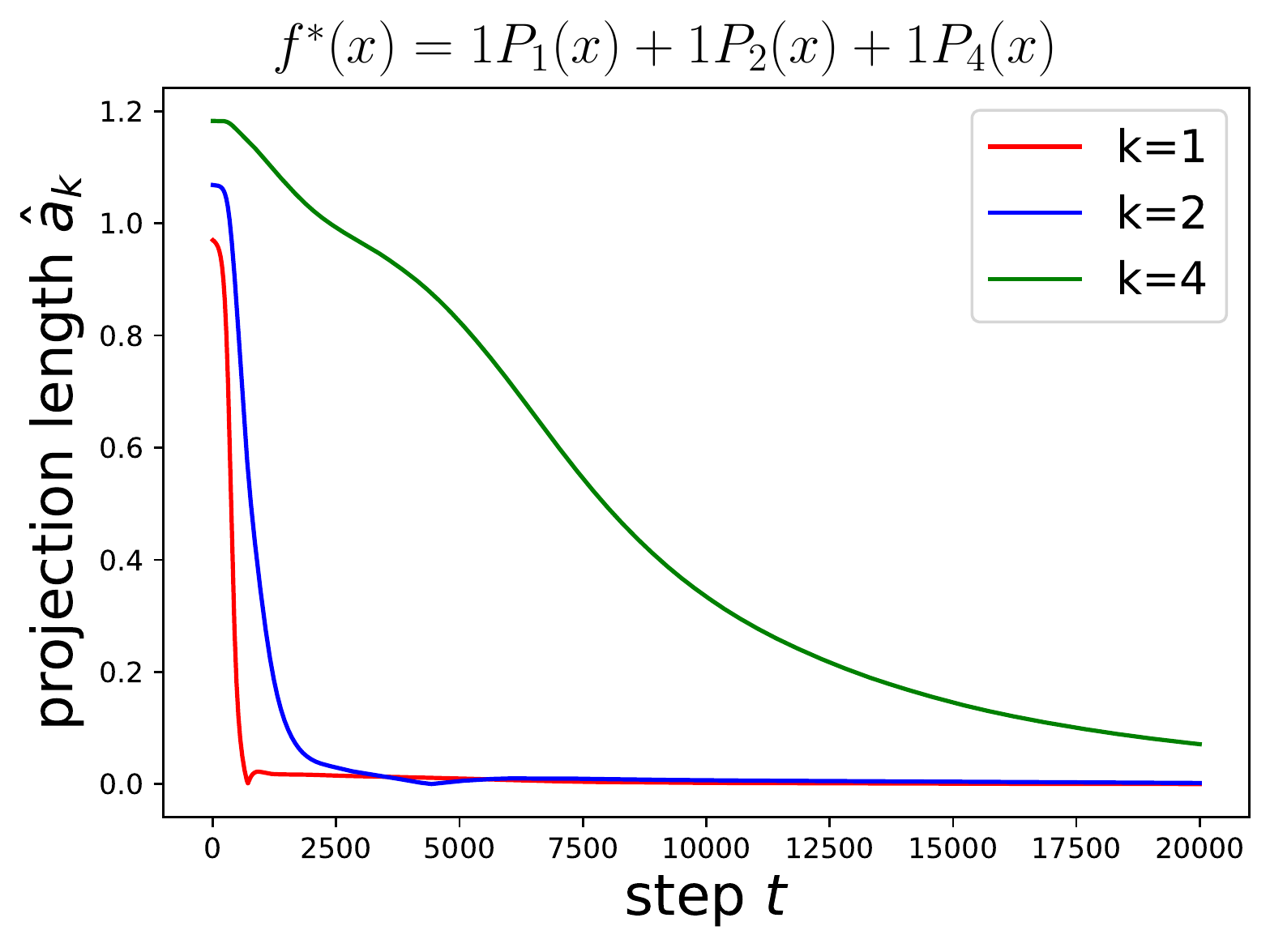}}
     \subfigure[different scale]{\includegraphics[width=0.47\textwidth]{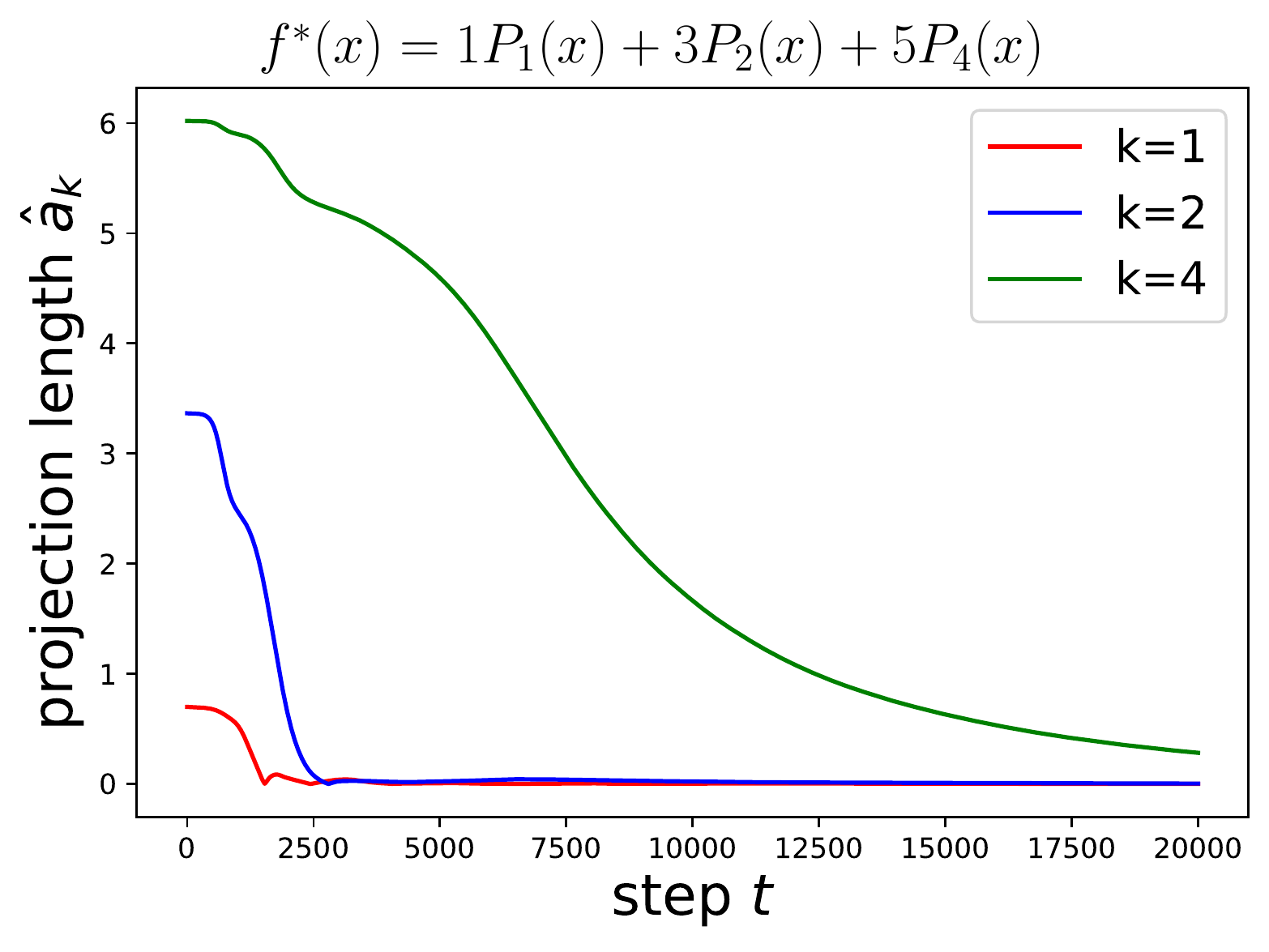}}
    \caption{Convergence curve of projection lengths for the piece-wise uniform distribution example.}
    \label{fig1-2}
\end{figure}

\textbf{Normalized non-isotropic Gaussian} We generate the data by first generating non-zero mean, non-isotropic Gaussian vectors, and then normalize them to unit length. 
The Gaussian mean vector is generated elementwise from $\text{Unif}([-1,1])$; the covariance matrix is generated as $\bSigma = \Ab^{\top} \Ab$, where $\Ab \in \RR_{d \times D}$($d=10, D=20$) has i.i.d. standard Gaussian entries. The results are shown in Figure~\ref{fig1-3}.

\begin{figure}[H]
     \centering
     \subfigure[same scale]{\includegraphics[width=0.47\textwidth]{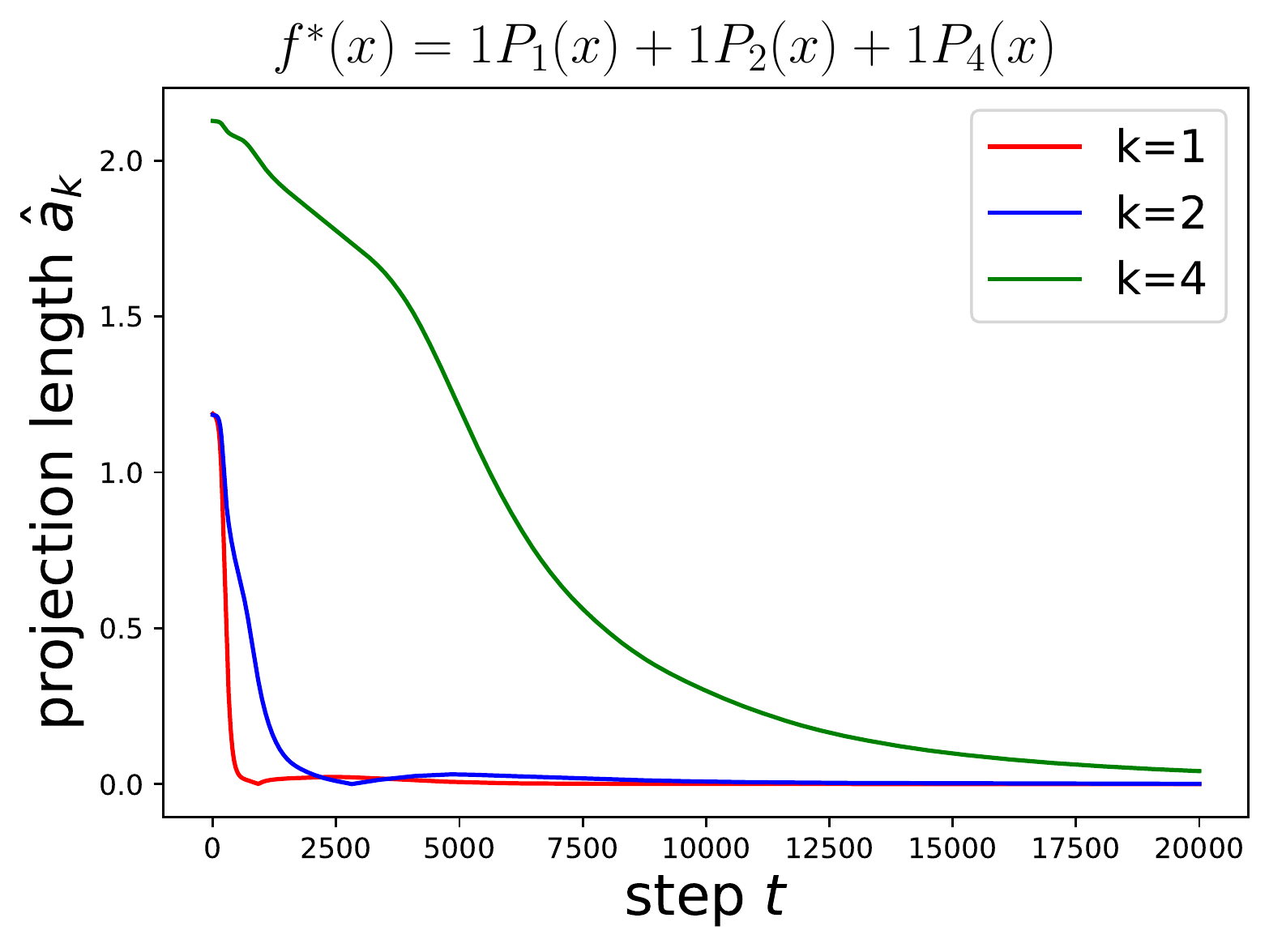}}
     \subfigure[different scale]{\includegraphics[width=0.47\textwidth]{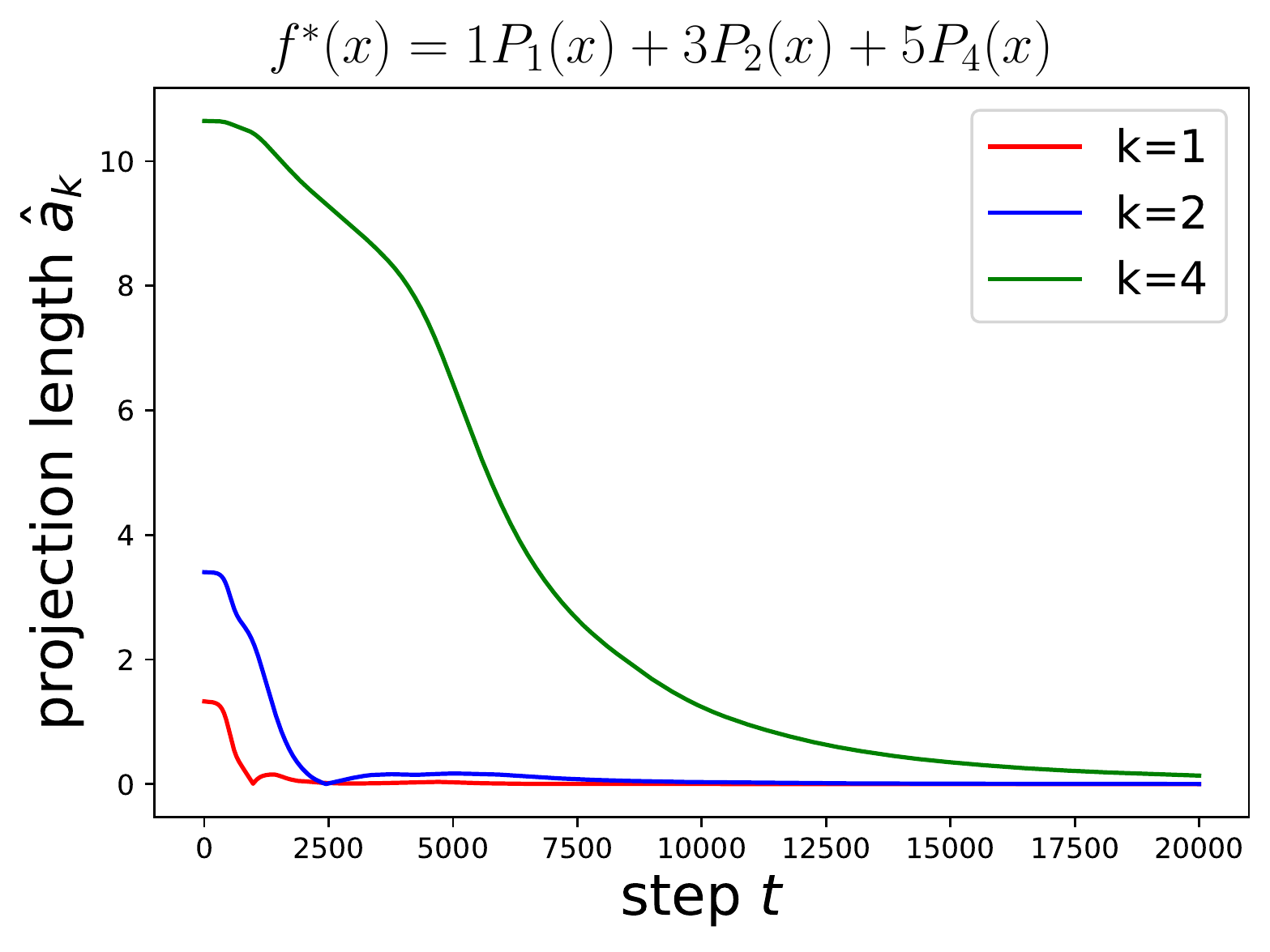}}
    \caption{Convergence curve for projection lengths for the normalized non-isotropic Gaussian example.}
    \label{fig1-3}
\end{figure}

\textbf{Normalized Gaussian mixture} The inputs are drawn from a mixture of 3 non-isotropic Gaussian distributions described above in the second example. 
The results are shown in Figure~\ref{fig1-4}.

Figures~\ref{fig1-2}, \ref{fig1-3}, \ref{fig1-4} show that the residual components corresponding to lower order polynomials are still learned relatively faster, even for the non-uniform distributions described above, where spherical harmonics are no longer exactly the eigenfunctions of NTK. This suggests that our theoretical results  can tolerate certain level of model misspecification, and therefore the spectral bias characterized in Theorem~\ref{thm:projectionconvergence} and Corollaries~\ref{col:largek}, \ref{col:larged} holds in a variety of problem settings. 

\begin{figure}[H]
     \centering
     \subfigure[same scale]{\includegraphics[width=0.47\textwidth]{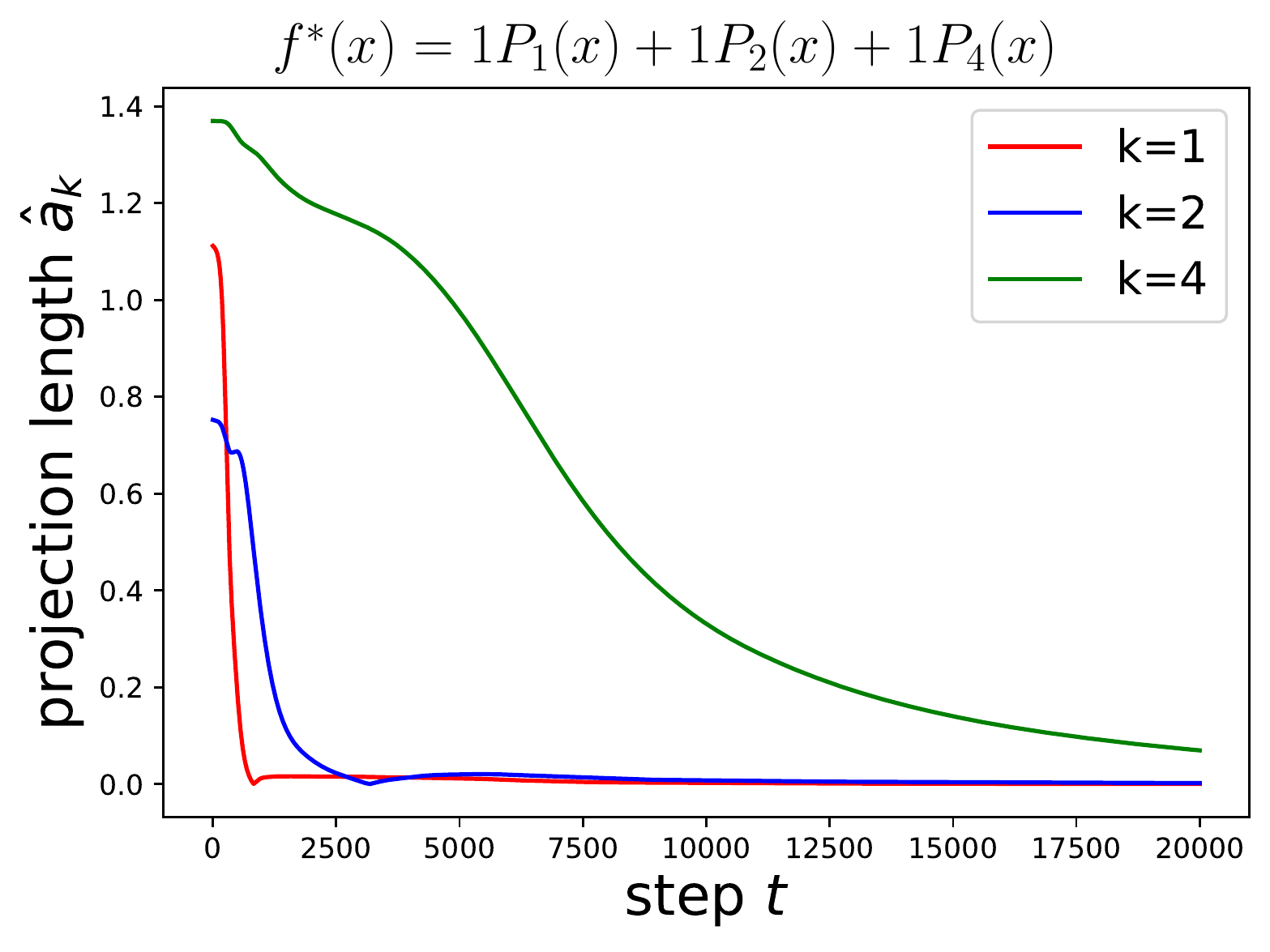}}
     \subfigure[different scale]{\includegraphics[width=0.47\textwidth]{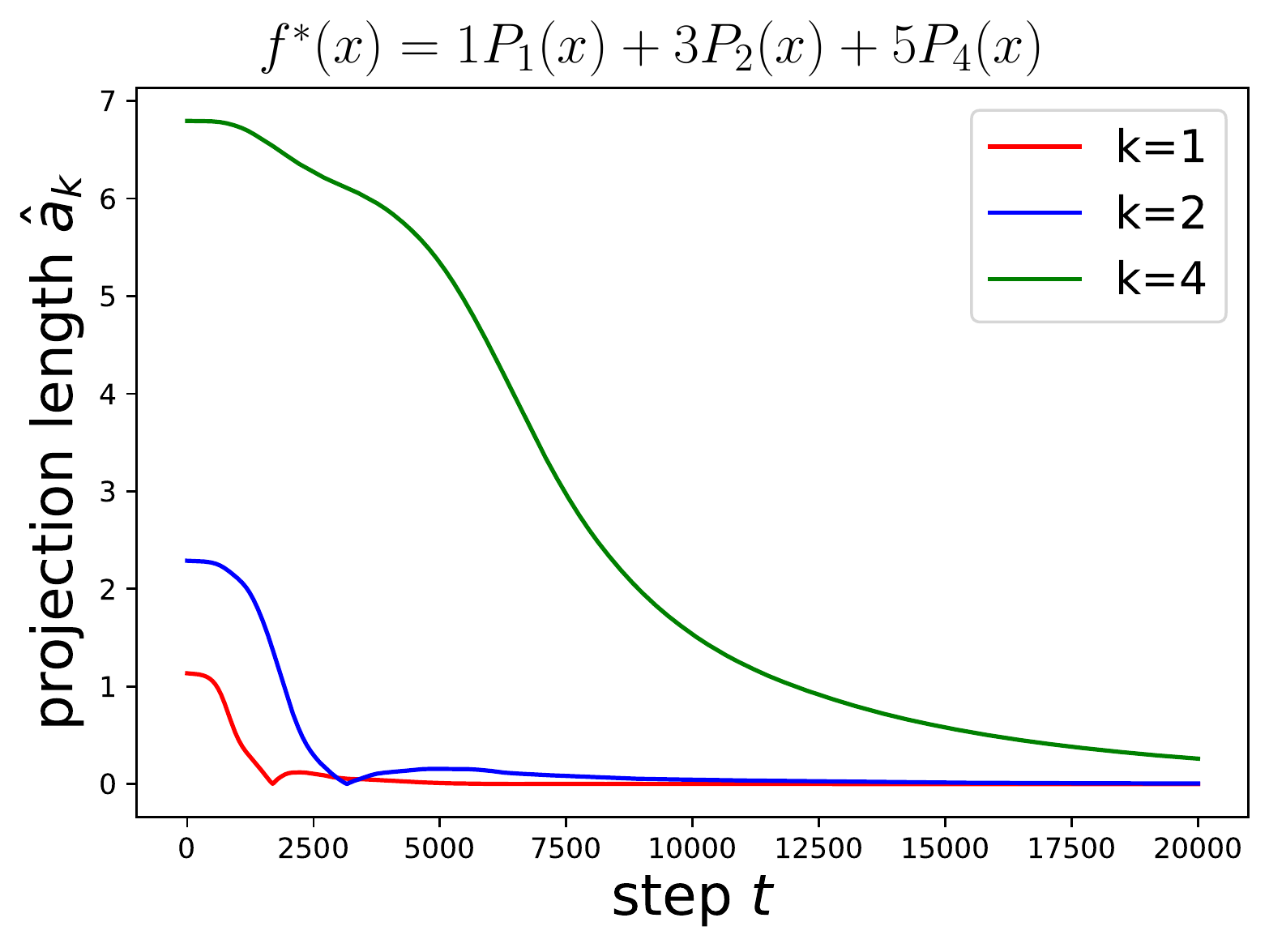}}
    \caption{Convergence curve of projection lengths for the normalized Gaussian mixture example.}
    \label{fig1-4}
\end{figure}






\section{Conclusion}

In this paper, we give theoretical justification for spectral bias through a detailed analysis of the convergence behavior of two-layer ReLU networks. We show that the convergence of gradient descent in different directions depends on the corresponding eigenvalues and essentially exhibits different convergence rates. We show Mercer decomposition of neural tangent kernel and give explicit order of eigenvalues of integral operator with respect to the neural tangent kernel when the data is uniformly distributed on the unit sphere $\SSS^{d}$. Combined with the convergence analysis, we give the exact order of convergence rate on different directions. We also conduct experiments on synthetic data to support our theoretical result.

\appendix


\section{Appendix A: A Review on Spherical Harmonics}
In this section, we give a brief review on relevant concepts in spherical harmonics. For more detials, see \citet{bach2017harmonics,bietti2019inductive,frye2012spherical,atkinson2012spherical} for references.

We consider the unit sphere $\SSS^{d} = \cbr{\xb \in \RR^{d+1}: \|\xb\|_2=1}$, whose surface area is given by $\omega_d = 2 \pi^{(d+1)/2}/\Gamma((d+1)/2)$ and denote $\tau_d$ the uniform measure on the sphere. For any $ k \geq 1 $, we consider a set of spherical harmonics 
$$
\cbr{Y_{k,j}: \SSS^{d} \rightarrow \RR | 1 \le j \le N(d,k)=\frac{2k+d-1}{k} {\binom{k+d-2}{d-1}} }.
$$ 
They form an orthonormal basis and satisfy the following equation $\langle Y_{ki} , Y_{sj} \rangle_{\SSS^d} = \int_{\SSS^d} Y_{ki}(x)  Y_{sj}(x) d\tau_d(x) = \delta_{ij} \delta_{sk}$. Moreover, since they are homogeneous functions of degree $k$, it is clear that $Y_k(x)$ has the same parity as $k$.

We have the addition formula
\begin{align}\label{eq:additionformula}
    \sum_{j=1}^{N(d,k)} Y_{k,j}(\xb)Y_{k,j}(\yb) = N(d,k) P_k(\dotp{\xb}{\yb}),
\end{align}
where $P_k(t)$ is the Legendre polynomial of degree $k$ in $d+1$ dimensions, explicitly given by (Rodrigues' formula)
\begin{align*}
   P_k(t) = \left(-\frac{1}{2}\right)^k   \frac{\Gamma\left(\frac{d}{2}\right)}{\Gamma\left(k + \frac{d}{2}\right)}
   \left(1-t^2\right)^{\frac{2-d}{2}}  \left(\frac{d}{dt}\right)^k
   \left(1-t^2\right)^{k + \frac{d-2}{2}}.
\end{align*}
We can also see that $P_k(t)$, the Legendre polynomial of degree $k$  shares the same parity with $k$. By the orthogonality and the addition formula (\ref{eq:additionformula}) we have,
\begin{align}\label{equation:ppintegral}
    \int_{\SSS^d} P_j(\dotp{\wb}{\xb})P_k(\dotp{\wb}{\yb}) d \tau_d(\wb) = 
    \frac{\delta_{jk}}{N(d,k)}P_k(\dotp{\xb}{\yb}).
\end{align}

The following recurrence relation holds for Legendre polynomials:
\begin{align}\label{equation:precurrence}
    tP_k(t) = \frac{k}{2k+d-1} P_{k-1}(t) +
    \frac{k+d-1}{2k+d-1} P_{k+1}(t),
\end{align}
for $k \ge 1$ and $tP_0(t) = P_1(t)$ for $k=0$.

The Hecke-Funk formula is given for a spherical harmonic $Y_k$ of degree $k$ as follows.
\begin{align}\label{eq:heckefunk}
    \int_{\SSS^d} f(\dotp{\xb}{\yb}) Y_k(\yb) d \tau_d(\yb)
    =
    \frac{\omega_{d-1}}{\omega_{d}}  Y_k(\xb) \int_{-1}^{1} f(t) P_k(t) (1-t^2)^{(d-2)/2}dt.
\end{align}

\section{Appendix B: Proof of Main Theorems}\label{sec:proof_of_lemmas_1}

\subsection{Proof of Lemma~\ref{lemma:projectionconcentration}}
Here we provide the proof of Lemma~\ref{lemma:projectionconcentration}, which is based on direct application of concentration inequalities.

\begin{proof}[Proof of Lemma~\ref{lemma:projectionconcentration}]
By definition, we have
\begin{align*}
    (\Vb_{r_k}^\top \Vb_{r_k} )_{ij} = \vb_i^\top \vb_j = \frac{1}{n} \sum_{s=1}^n \phi_i(\xb_s)\phi_j(\xb_s),
\end{align*}
where $\xb_1,\ldots \xb_n$ are i.i.d. samples from distribution $\tau$. Note that $\phi_i(\xb)$'s are orthonormal functions in $L_\tau^2(\SSS^{d})$. By definition we have
\begin{align*}
    \EE \Bigg[  \frac{1}{n} \sum_{s=1}^n \phi_i(\xb_s)\phi_j(\xb_s) \Bigg] = \int \phi_i(\xb)\phi_j(\xb) d \tau(\xb) = \delta_{i,j},
\end{align*}
where $\delta_{i,j} = 1$ if $i = j$ and $\delta_{i,j} = 0$ otherwise. Now since $\phi_i(\xb) \leq M $ for all $\xb\in \SSS^{d}$ and $i\in[r_k]$, by standard Hoeffding's inequality (see Proposition 5.10 in \citet{vershynin2010introduction}), with probability at least $1 - \delta / r_k^2 $ we have
\begin{align*}
    \big|(\Vb_{r_k}^\top \Vb_{r_k} )_{ij} - \delta_{i,j} \big| \leq C_1 M^2 \sqrt{\frac{ \log(r_k^2 / \delta) }{n}} \leq C_2 M^2 \sqrt{\frac{ \log(r_k / \delta) }{n}} ,
\end{align*}
where $C_1,C_2$ are absolute constants. Applying a union bound over all $i,j \in [r_k]\times [r_k]$ finishes the proof.
\end{proof}

\subsection{Proof of Theorem~\ref{thm:projectionconvergence}}
In this section we give the proof of  Theorem~\ref{thm:projectionconvergence}. The core idea of our proof is to establish connections between neural network gradients throughout training and the neural tangent kernel. To do so, we first introduce the following definitions and notations.

Define $\Kb^{(0)} = m^{-1} ( \la \nabla_{\Wb} f_{\Wb^{(0)}}(\xb_i) , \nabla_{\Wb} f_{\Wb^{(0)}}(\xb_j) \ra )_{n\times n}$, $\Kb^{(\infty)} = (\kappa(\xb_i,\xb_j))_{n\times n}= \lim_{m\rightarrow \infty} \Kb^{(0)}$.
Let $\{ \hat\lambda_i \}_{i=1}^n$, $\hat\lambda_1 \geq\cdots \geq \hat\lambda_n$ be the eigenvalues of $n^{-1} \Kb^\infty$, and $\hat\vb_1,\ldots,\hat\vb_n$ be the corresponding eigenvectors. Set $\hat\Vb_{{r_k}} = (\hat\vb_1,\ldots,\hat\vb_{{r_k}})$, $\hat\Vb_{{r_k}}^\bot = (\hat\vb_{{r_k} + 1},\ldots,\hat\vb_{n})$. 
For notation simplicity, we denote $\nabla_\Wb f_{\Wb^{(0)}}(\xb_i) = [\nabla_{\Wb} f_{\Wb}(\xb_i)]\big|_{\Wb = \Wb^{(0)}}$, $\nabla_{\Wb_l} f_{\Wb^{(0)}}(\xb_i) = [\nabla_{\Wb_l} f_{\Wb}(\xb_i)]\big|_{\Wb = \Wb^{(0)}}$, $l=1,2$.

The following lemma's purpose is to further connect the eigenfunctions of NTK with their finite-width, finite-sample counterparts. Its first statement is proved in \citet{su2019learning}.

\begin{lemma}\label{lemma:projectiondifference}
Suppose that $ |\phi_i(\xb) | \leq M$ for all $\xb\in S^{d}$. There exist absolute constants $C, C',c''>0$, such that for any $\delta >0$ and integer $k$ with ${r_k} \leq n$, if $n \geq C  (\lambda_{{r_k}} - \lambda_{{r_k} + 1})^{-2} \log(1/\delta)$, then with probability at least 
$1 - \delta$, 
\begin{align*}
    &\| \Vb_{{r_k}}^\top \hat\Vb_{{r_k}}^\bot \|_F \leq C'\frac{1}{\lambda_{r_k} - \lambda_{r_k + 1}} \cdot \sqrt{\frac{ \log(1/\delta)}{n} },\\
    &\| \Vb_{{r_k}} \Vb_{{r_k}}^\top -  \hat\Vb_{{r_k}} \hat\Vb_{{r_k}}^\top \|_2 \leq C''\bigg[  \frac{1}{\lambda_{r_k} - \lambda_{r_k + 1}} \cdot \sqrt{\frac{ \log(1/\delta)}{n} }  + M^2 r_k \sqrt{\frac{ \log(r_k / \delta)}{ n}} \bigg].
\end{align*}
\end{lemma}

The following two lemmas gives some preliminary bounds on the function value and gradients of the neural network around random initialization. They are proved in \citet{cao2019generalizationsgd}.

\begin{lemma}[\citet{cao2019generalizationsgd}]\label{lemma:initialfunctionvaluebound}
For any $\delta > 0$, if $m\geq C\log(n/\delta)$ for a large enough absolute constant $C$, then with probability at least $1 - \delta$, 
$ |f_{\Wb^{(0)}} (\bx_i)| \leq \cO(\sqrt{\log( n / \delta)}) $
for all $i\in[n]$.
\end{lemma}

\begin{lemma}[\citet{cao2019generalizationsgd}]\label{lemma:NNgradient_uppbound}
There exists an absolute constant $C$ such that, with probability at least $1 - \cO(n) \cdot \exp[-\Omega(m\omega^{2/3} )] $, for all $i\in [n]$, $l\in[L]$ and $\Wb\in \cB(\Wb^{(0)},\omega)$ with $ \omega \leq C [\log(m)]^{-3}$, it holds uniformly that 
\begin{align*}
    \| \nabla_{\Wb_l} f_{\Wb}(\xb_i) \|_F \leq \cO(\sqrt{m}).
\end{align*}
\end{lemma}

The following lemma is the key to characterize the dynamics of the residual throughout training. These bounds in Lemma~\ref{lemma:explicitboundinsideball} are the ones that distinguish our analysis from previous works on neural network training in the NTK regime \citep{du2018gradient,su2019learning}, since our analysis provides more careful characterization on the residual along different directions.



\begin{lemma}\label{lemma:explicitboundinsideball}
Suppose that the iterates of gradient descent $\Wb^{(0)},\ldots, \Wb^{(t)}$ are inside the ball $\cB(\Wb^{(0)}, \omega )$. If $\omega\leq \tilde\cO(\min\{[\log(m)]^{-3/2}, \lambda_{r_k}^{3}, (\eta m)^{-3}\} )$ and $ n \geq \tilde\cO(\lambda_{r_k}^{-2})$, then with probability at least $1 -\cO(t^2n^2)\cdot \exp[-\Omega(m \omega^{2/3})]$,
\begin{align}
    \| (\hat\Vb_{r_k}^{\bot})^\top (\yb - \hat\yb^{({t'})}) \|_2 &\leq \| (\hat\Vb_{r_k}^{\bot})^\top (\yb - \hat\yb^{(0)}) \|_2 +  {t'} \cdot \omega^{1/3}\eta m \theta^2 \cdot \sqrt{n}\cdot  \tilde\cO(1 + \omega \sqrt{m})\label{eq:lemma_explicitboundinsideball_eq1}\\
    \| \hat\Vb_{r_k}^\top (\yb - \hat\yb^{({t'})}) \|_2 &\leq ( 1 - \eta m \theta^2 \lambda_{r_k} /2)^{t'} \| \hat\Vb_{r_k}^\top (\yb - \hat\yb^{(0)}) \|_2 + {t'} \lambda_{r_k}^{-1} \cdot \omega^{2/3} \eta m \theta^2\cdot \sqrt{n}\cdot  \tilde\cO(1 + \omega \sqrt{m}) \nonumber\\
    &\quad + \lambda_{r_k}^{-1} \cdot \tilde\cO( \omega^{1/3} )\cdot  \| (\hat\Vb_{r_k}^\bot)^\top (\yb - \hat\yb^{(0)}) \|_2\label{eq:lemma_explicitboundinsideball_eq2}\\
    \| \yb - \hat\yb^{({t'})} \|_2 & \leq \tilde\cO(\sqrt{n})
    \cdot (1 - \eta m \theta^2 \lambda_{r_k} / 2)^{t'}  + \tilde\cO (\sqrt{n}\cdot (\eta m \theta^2 \lambda_{r_k})^{-1} ) \nonumber \\
    &\quad + \lambda_{r_k}^{-1} {t'}  \omega^{1/3}\cdot \sqrt{n} \cdot  \tilde\cO(1 + \omega \sqrt{m}) \label{eq:lemma_explicitboundinsideball_eq4}
\end{align}
for all ${t'}=0,\ldots, t-1$.
\end{lemma}

Now we are ready to prove Theorem~\ref{thm:projectionconvergence}. To illustrate the intuition behind the proof,  the proof can be mainly summarized into three steps:
(i) Using a induction argument together with Lemma~\ref{lemma:explicitboundinsideball} to show that under the theorem assumptions, all $T$ gradient descent iterates $\Wb^{(0)},\ldots,\Wb^{(T)}$ are inside the ball $\cB(\Wb^{(0)},\omega)$ with $\omega = O( T / (\btheta \lambda_{r_k} \sqrt{m}) )$. (ii) We then reapply Lemma~\ref{lemma:explicitboundinsideball} with $t = T$ and obtain a bound on $\| \hat\Vb_{r_k}^\top (\yb - \hat\yb^{(T)}) \|_2$. (iii) At last, we  utilize the concentration results given in Lemma~\ref{lemma:projectiondifference} to finalize the proof. The detaled proof is given as follows.

\begin{proof}[Proof of Theorem~\ref{thm:projectionconvergence}]
Define $\omega = \overline{C} T / (\btheta \lambda_{r_k} \sqrt{m})$ for some small enough absolute constant $\overline{C}$. Then by union bound, as long as $m \geq \tilde\Omega( \poly(T, \lambda_{r_k}^{-1},\epsilon^{-1}) )$, the conditions on $\omega$ given in Lemmas~\ref{lemma:NNgradient_uppbound} and \ref{lemma:explicitboundinsideball} are satisfied. 

We first show that all the iterates $\Wb^{(0)},\ldots,\Wb^{(T)}$ are inside the ball $\cB(\Wb^{(0)},\omega)$. We prove this result by inductively show that $\Wb^{(t)} \in \cB(\Wb^{(0)},\omega)$, $t=0,\ldots, T$. First of all, it is clear that $\Wb^{(0)} \in \cB(\Wb^{(0)},\omega)$. Suppose that $\Wb^{(0)},\ldots,\Wb^{(t)} \in \cB(\Wb^{(0)},\omega )$. 
Then the results of Lemmas~\ref{lemma:NNgradient_uppbound} and \ref{lemma:explicitboundinsideball} hold for $\Wb^{(0)},\ldots, \Wb^{(t)}$. Denote $\ub^{(t)} = \yb - \hat\yb^{(t)}$, $t\in {T}$. Then we have
\begin{align*}
    \| \Wb_l^{(t+1)} - \Wb_l^{(0)} \|_F &\leq \sum_{{t'} = 0}^t \| \Wb_l^{({t'} + 1)} - \Wb_l^{({t'})} \|_F \\
    &= \eta \sum_{{t'} = 0}^t \Bigg\| \frac{1}{n} \sum_{i=1}^n (y_i - \theta\cdot f_{\Wb^{(t)}}(\xb_i) )\cdot  \theta\cdot  \nabla_{\Wb_l} f_{\Wb^{(t)}}(\xb_i) \Bigg\|_F\\
    &\leq \eta \theta \sum_{{t'} = 0}^t \frac{1}{n} \sum_{i=1}^n |y_i - \theta\cdot  f_{\Wb^{(t)}}(\xb_i) |\cdot \| \nabla_{\Wb_l} f_{\Wb^{(t)}}(\xb_i) \|_F\\
    &\leq C_1 \eta \theta \sqrt{m} \sum_{{t'} = 0}^t \frac{1}{n} \sum_{i=1}^n |y_i - \theta\cdot f_{\Wb^{(t)}}(\xb_i) |\\
    &\leq C_1 \eta \theta \sqrt{m/n} \sum_{{t'} = 0}^t \| \yb -\hat\yb^{({t'})} \|_2,
\end{align*}
where the second inequality follows by Lemma~\ref{lemma:NNgradient_uppbound}. 
By Lemma~\ref{lemma:explicitboundinsideball}, we have
\begin{align*}
    \sum_{{t'} = 0}^t \| \yb -\hat\yb^{({t'})} \|_2 &\leq  \tilde\cO(\sqrt{n}/ (\eta m \theta^2 \lambda_{r_k} ) )  + \tilde\cO (T \sqrt{n} / (\eta m \theta^2 \lambda_{r_k}) ) + \lambda_{r_k}^{-1} T^2  \omega^{1/3} \sqrt{n} \cdot  \tilde\cO(1 + \omega \sqrt{m}).
\end{align*}
It then follows by the choice $\omega = \overline{C} T / ( \theta \lambda_{r_k} \sqrt{m})$, $\eta = \tilde\cO ( (m \theta^2 )^{-1} )$, $\theta = \tilde\cO(\epsilon)$ and the assumption $m \geq \tilde\cO(\poly( T, \lambda_{r_k}^{-1}, \epsilon^{-1}))$ that $\| \Wb_l^{(t+1)} - \Wb_l^{(0)} \|_F \leq \omega$, $l=1,2$. Therefore by induction, we see that with probability at least $1 - \cO(T^3n^2)\cdot \exp[-\Omega(m \omega^{2/3})]$, $\Wb{(0)},\ldots, \Wb{(T)} \in \cB(\Wb^{(0)},\omega)$ . 

Applying Lemma~\ref{lemma:explicitboundinsideball} then gives
\begin{align*}
    n^{-1/2}\cdot\| \hat\Vb_{r_k}^\top (\yb - \hat\yb^{(T)}) \|_2 &\leq  ( 1 - \eta m \theta^2 \lambda_{r_k}/2)^T\cdot n^{-1/2}\cdot\| \hat\Vb_{r_k}^\top (\yb - \hat\yb^{(0)}) \|_2 \\
    &\quad + T\lambda_{r_k}^{-1} \cdot \omega^{2/3} \eta m \theta^2 \cdot  \tilde\cO(1 + \omega \sqrt{m}) \nonumber\\
    &\quad + \lambda_{r_k}^{-1} \cdot \tilde\cO( \omega^{1/3} )\cdot n^{-1/2}\cdot \| (\hat\Vb_{r_k}^\bot)^\top (\yb - \hat\yb^{(0)}) \|_2.
\end{align*}
Now by $\omega = \overline{C} T/(\lambda_{r_k} \sqrt{m}) $, $\eta = \tilde\cO(\theta^2 m)^{-1}$ and the assumption that $m \geq m^* = \tilde\cO( \lambda_{r_k}^{-14} \cdot \epsilon^{-6} )$, we obtain 
\begin{align}\label{eq:mainthm_proof_eq1}
    n^{-1/2}\cdot\| \hat\Vb_{r_k}^\top (\yb - \hat\yb^{(T)}) \|_2 \leq ( 1  - \lambda_{r_k})^T \cdot n^{-1/2}\cdot\| \hat\Vb_{r_k}^\top (\yb - \hat\yb^{(0)}) \|_2 + \epsilon / 16.
\end{align}
By Lemma~\ref{lemma:projectionconcentration}, $\theta = \tilde\cO(\epsilon)$ and the assumptions $n \geq \tilde\Omega( \max\{ \epsilon^{-1} ( \lambda_{r_k } - \lambda_{r_k + 1}  )^{-1} , \epsilon^{-2} M^2 r_k^2 \}   )$, the eigenvalues of $\Vb_{r_k}^\top \Vb_{r_k} $ are all between $1/\sqrt{2}$ and $\sqrt{2}$. Therefore we have
\begin{align*}
    \| \hat\Vb_{r_k}^\top (\yb - \hat\yb^{(T)}) \|_2 &= \| \hat\Vb_{r_k} \hat\Vb_{r_k}^\top (\yb - \hat\yb^{(T)}) \|_2\\ 
    &\geq \| \Vb_{r_k} \Vb_{r_k}^\top (\yb - \hat\yb^{(T)}) \|_2 - \| (\Vb_{r_k} \Vb_{r_k}^\top - \hat\Vb_{r_k} \hat\Vb_{r_k}^\top ) (\yb - \hat\yb^{(T)}) \|_2\\
    &\geq  \| \Vb_{r_k}^\top (\yb - \hat\yb^{(T)}) \|_2 / \sqrt{2} - \sqrt{n}\cdot \cO\bigg(  \frac{1}{\lambda_{r_k} - \lambda_{r_k + 1}} \cdot \sqrt{\frac{ \log(1/\delta)}{n} }  + M r_k \sqrt{\frac{ \log(r_k / \delta)}{ n}} \bigg)\\
    &\geq \| \Vb_{r_k}^\top (\yb - \hat\yb^{(T)}) \|_2 / \sqrt{2} - \epsilon \sqrt{n} / 16,
\end{align*}
where the second inequality follows  by Lemma~\ref{lemma:projectiondifference} and the fact $\Vb_{r_k}^\top \Vb_{r_k} \succeq (1/\sqrt{2})\Ib$. 
Similarly, 
\begin{align*}
    \| \hat\Vb_{r_k}^\top (\yb - \hat\yb^{(0)}) \|_2 &\leq \sqrt{2} \cdot \| \Vb_{r_k}^\top (\yb - \hat\yb^{(0)}) \|_2  + \epsilon \sqrt{n} / 16 \\
    &\leq  \sqrt{2} \cdot \| \Vb_{r_k}^\top \yb  \|_2 +  \sqrt{2} \cdot \| \Vb_{r_k}^\top \hat\yb^{(0)} \|_2 + \epsilon \sqrt{n} / 16.
\end{align*}
By Lemma \ref{lemma:projectionconcentration}, with probability at least $1 - \delta$, $ \| \Vb_{r_k}^\top \|_2 \leq 1 + C r_k M^2 \sqrt{\log(r_k / \delta) / n}$. Combining this result with Lemma \ref{lemma:initialfunctionvaluebound} gives $\| \Vb_{r_k}^\top \hat\yb^{(0)} \|_2 \leq \theta \cO(\sqrt{n \log{(n/ \delta)}}) \leq \epsilon \sqrt{n} / 8$.
Plugging the above estimates into \eqref{eq:mainthm_proof_eq1} gives
\begin{align*}
    n^{-1/2}\cdot \| \Vb_{r_k}^\top (\yb - \hat\yb^{(T)}) \|_2 \leq 2 ( 1  - \lambda_{r_k})^T\cdot n^{-1/2} \cdot \| \Vb_{r_k}^\top \yb  \|_2 + \epsilon.
\end{align*}
Applying union bounds completes the proof.
\end{proof}

\subsection{Proof of Theorem \ref{theorem:spectralanalysis}}
\begin{proof}[Proof of Theorem \ref{theorem:spectralanalysis}]
The idea of the proof is close to that of Proposition 5 in \citep{bietti2019inductive} where they consider $k \gg d$ and we present a more general case including $k \gg d$ and $d \gg k$. \\
For any function $g : \SSS^d \rightarrow \RR$, by denoting $g_0(\xb) = \int_{\SSS^d} g(\yb) d\tau_d(\yb)$, it can be decomposed as
\begin{align}
    g(\xb) = \sum_{k=0}^{\infty} g_k(\xb) 
          = \sum_{k=0}^{\infty} \sum_{j=1}^{N(d,k)} 
            \int_{\SSS^d} Y_{kj}(\yb)Y_{kj}(\xb)g(\yb) d\tau_d(\yb), \label{euqation:L2decomposition} 
\end{align}
where we project function $g$ to spherical harmonics.  
For a positive-definite dot-product kernel $\kappa(\xb,\xb') : \SSS^{d} \times \SSS^{d} \rightarrow \RR$ which has the form $\kappa(\xb,\xb') = \hat{\kappa}(\dotp{\xb}{\xb'})$ for $\hat{\kappa}  : [-1,1] \rightarrow \RR$, we obtain the following decomposition by (\ref{euqation:L2decomposition}) 
\begin{align*}
    \kappa(\xb,\xb') &= 
    \sum_{k=0}^{\infty} \sum_{j=1}^{N(d,k)} 
            \int_{\SSS^d} Y_{kj}(\yb)Y_{kj}(\xb) \hat{\kappa}(\dotp{\yb}{\xb'}) d\tau_d(\yb)\\
    &= \sum_{k=0}^{\infty} N(d,k) \frac{\omega_{d-1}}{\omega_{d}} P_k(\dotp{\xb}{\xb'}) \int_{-1}^{1} \hat{\kappa}(t) P_k(t) (1-t^2)^{(d-2)/2}dt, 
\end{align*}
where we apply the Hecke-Funk formula \eqref{eq:heckefunk} and addition formula \eqref{eq:additionformula}. Denote
$$\mu_k =  \left(\omega_{d-1}/\omega_{d}\right)  \int_{-1}^{1} \hat{\kappa}(t) P_k(t) (1-t^2)^{(d-2)/2} dt. $$ 
Then by the addition formula, we have
\begin{align}\label{Mercerdecomposition}
    \kappa(\xb,\xb')
    = \sum_{k=0}^{\infty} \mu_k
    N(d,k) P_k(\dotp{\xb}{\xb'}) 
    = \sum_{k=0}^{\infty} \mu_k
    \sum_{j=1}^{N(p,k)}Y_{k,j}(\xb) Y_{k,j}(\xb').
\end{align}
(\ref{Mercerdecomposition}) is the Mercer decomposition for the kernel function $\kappa(\xb,\xb')$ and $\mu_k$ is exactly the eigenvalue of the integral operator $L_K$ on ${L}_2(\SSS^d)$ defined by 
\begin{equation*}
\begin{aligned}
L_\kappa(f)(\yb)= \int_{\SSS^d} \kappa(\xb,\yb) f(\xb) d\tau_d(\xb),  ~~~ f \in {L}_2(\SSS^d).
\end{aligned}
\end{equation*}

By using same technique as $\kappa(\xb,\xb')$, we can derive a similar expression for $\sigma(\dotp{\wb}{\xb}) =\max\left\{\dotp{\wb}{\xb},0\right\}$ and $\sigma'(\dotp{\wb}{\xb}) =\ind\{\dotp{\wb}{\xb}>0\}$, since they are essentially dot-product function on ${L}_2(\SSS^d)$. We deliver the expression below without presenting proofs.
\begin{align}
   &\sigma'(\dotp{\wb}{\xb}) 
    = \sum_{k=0}^{\infty} \beta_{1,k}
    N(d,k) P_k(\dotp{\wb}{\xb}), \label{equation:sigma}\\
   &\sigma(\dotp{\wb}{\xb}) 
    = \sum_{k=0}^{\infty} \beta_{2,k}
    N(d,k) P_k(\dotp{\wb}{\xb}), \label{equation:sigmaprime}
\end{align}
where $\beta_{1,k} =  \left(\omega_{d-1}/\omega_{d} \right) \int_{-1}^{1} \sigma(t)  P_k(t) (1-t^2)^{(d-2)/2} dt$ and $\beta_{2,k} =  \left(\omega_{d-1}/\omega_{d} \right) \int_{-1}^{1} \sigma'(t)  P_k(t) (1-t^2)^{(d-2)/2} dt$.
We add more comments on the values of $\beta_{1,k}$ and $\beta_{2,k}$. It has been pointed out in \cite{bach2017harmonics} that when $k > \alpha$ and when $k$ and $\alpha$ have same parity, we have $\beta_{\alpha+1,k} = 0$. This is because the Legendre polynomial $P_k(t)$ is orthogonal to any other polynomials of degree less than $k$ with respect to the density function $p(t) = (1-t^2)^{(d-2)/2}$. Then we clearly know that $\beta_{1,k} = 0$ for $k =2j$ and $\beta_{2,k} = 0$ for $k =2j+1$ with $j \in \NN^+$. 

For two kernel function defined in (\ref{definition:kernel}), we have 
\begin{align}
   \kappa_1(\xb,\xb') 
    &= \EE_{\wb\sim N(\mathbf{0}, \Ib)} \left[\sigma'(\la \wb,\xb \ra)\sigma'(\la \wb,\xb' \ra)\right] \nonumber\\ 
    &= \EE_{\wb\sim N(\mathbf{0}, \Ib)} \left[\sigma'(\la \wb / \left\|\wb\right\|_2,\xb \ra)\sigma'(\la \wb /  \left\|\wb\right\|_2,\xb' \ra)\right] \nonumber\\
    &= \int_{\SSS^d} \sigma'(\dotp{\vb}{\xb}) \sigma'(\dotp{\vb}{\xb'}) d \tau_d(\vb).\label{equation:kappa1integral}
\end{align}
The first equality holds because $\sigma'$ is 0-homogeneous function and the second equality is true since the normalized direction of a multivariate Gaussian random variable satisfies uniform distribution on the unit sphere.
Similarly we can derive 
\begin{align}\label{equation:kappa2integral}
   \kappa_2(\xb,\xb') 
   = (d+1) \int_{\SSS^d} \sigma(\dotp{\vb}{\xb}) \sigma(\dotp{\vb}{\xb'}) d \tau_d(\vb).
\end{align}
By combining (\ref{equation:ppintegral}), (\ref{equation:sigma}), (\ref{equation:sigmaprime}), (\ref{equation:kappa1integral}) and (\ref{equation:kappa2integral}), we can get

\begin{align}\label{equation:kappa1lambda}
   \kappa_1(\xb,\xb') 
   = \sum_{k=0}^\infty \beta_{1,k}^2 N(d,k) P_k(\dotp{\xb}{\xb'}),
\end{align}
and
\begin{align}\label{equation:kappa2lambda}
   \kappa_2(\xb,\xb') 
   = (d+1)\sum_{k=0}^\infty \beta_{2,k}^2 N(d,k) P_k(\dotp{\xb}{\xb'}).
\end{align}
Comparing (\ref{Mercerdecomposition}), (\ref{equation:kappa1lambda}) and (\ref{equation:kappa2lambda}), we can easily show that 
\begin{align}\label{euqation:muandlambda}
   \mu_{1,k} = \beta_{1,k}^2~~ \text{and}~~ \mu_{2,k} = (d+1)\beta_{2,k}^2.
\end{align}
In \cite{bach2017harmonics}, for $\alpha = 1,2$, explicit expressions of $\beta_{\alpha,k}$ for $k \geq \alpha + 1$ are presented as follows:
\begin{align*}
    \beta_{\alpha +1, k} 
    &= \frac{d-1}{2 \pi} \frac{\alpha! (-1)^{(k-1-\alpha)/2}}{2^k}
    \frac{\Gamma (d/2) \Gamma (k-\alpha)}{\Gamma(\frac{k-\alpha+1}{2}) \Gamma(\frac{k+d+\alpha+1}{2})}.
\end{align*}
By Stirling formula $\Gamma(x) \approx x^{x-1/2}e^{-x} \sqrt{2\pi}$, we have following expression of $\beta_{\alpha + 1, k}$ for $k \geq \alpha + 1$
\begin{align*}
    \beta_{\alpha + 1, k} = 
    C(\alpha) \frac{(d-1) d^{\frac{d-1}{2}}  (k-\alpha)^{k-\alpha -\frac{1}{2}}}{(k-\alpha+1)^{\frac{k-\alpha}{2}} (k+d+\alpha+1)^{\frac{k+d+\alpha}{2}}}
    = \Omega\left( d^{\frac{d+1}{2}} k^{\frac{k-\alpha-1}{2}} (k+d)^{\frac{-k-d-\alpha}{2}}\right)
\end{align*}
where $ C(\alpha)= \frac{\sqrt{2} \alpha!}{2\pi}  \exp\{\alpha +1\}$.
Also $\beta_{\alpha+1,0}= \frac{d-1}{4\pi} \frac{\Gamma{\left( \frac{\alpha +1}{2}\right)} \Gamma{\left(\frac{d}{2}\right)}}{\Gamma{\left(\frac{d+\alpha+2}{2}\right)}}$, $\beta_{1,1}=\frac{d-1}{2d\pi}$ and $\beta_{2,1} = \frac{d-1}{4\pi d} \frac{\Gamma{\left( \frac{1}{2}\right)} \Gamma{\left(\frac{d+2}{2}\right)}}{\Gamma{\left(\frac{d+3}{2}\right)}}$. Thus combine (\ref{euqation:muandlambda}) we know that $\mu_{\alpha + 1,k} = \Omega \left( d^{d+1+\alpha} k^{k-\alpha-1} (k+d)^{-k-d-\alpha}\right)$.\\
By considering (\ref{equation:precurrence}) and (\ref{Mercerdecomposition}), we have 
\begin{align*}
    \mu_0 = \mu_{1,1} + 2\mu_{2,0},~~~ \mu_{k'} = 0,k'=2j+1, ~j\in \NN^+,
\end{align*}
and
\begin{align*}
    \mu_k = \frac{k}{2k+d-1} \mu_{1,k-1} +
    \frac{k+d-1}{2k+d-1} \mu_{1,k+1} + 2\mu_{2,k},
\end{align*}
for $k \geq 1$ and $k \neq k'$.
From the discussion above, we thus know exactly that for $k \geq 1$
\begin{align*}
    \mu_k = \Omega \left(\max\left\{ d^{d+1} k^{k-1} (k+d)^{-k-d}, d^{d+1} k^k (k+d)^{-k-d-1}, d^{d+2} k^{k-2} (k+d)^{-k-d-1} \right\}\right).
\end{align*}
This finishes the proof.
\end{proof}

\subsection{Proof of Corollaries \ref{col:largek} and \ref{col:larged}}
\begin{proof}[Proof of Corollaries \ref{col:largek} and \ref{col:larged}]
We only need to bound  $ |\phi_j(\xb) | $ for $j\in [r_k]$ to finish the proof. Since now we assume input data follows uniform distribution on the unit sphere $\SSS^{d}$, $\phi_j(\xb)$ would be spherical harmonics of order at most $k$ for $j\in [r_k]$. For any spherical harmonics $Y_k$ of order $k$ and any point on $\SSS^d$, we have an upper bound (Proposition 4.16 in \cite{frye2012spherical})
\begin{align*}
    \left|Y_k(\xb)\right| \leq \left(N(d,k)    \int_{\SSS^d} Y_k^2(\yb) d \tau_d(\yb)\right)^\frac{1}{2}.
\end{align*}
Thus we know that $ |\phi_j(\xb) | \leq\sqrt{ N(d,k)}$. For $ k \gg d$, we have $N(d,k) = \frac{2k+d-1}{k} {\binom{k+d-2}{d-1}} = \cO(k^{d-1})$. For $ d \gg k$, we have $N(d,k) = \frac{2k+d-1}{k} {\binom{k+d-2}{d-1}} = \cO(d^k)$. This completes the proof.
\end{proof}

\section{Appendix C: Proof of Lemmas in Appendix B}\label{sec:proof_of_lemmas_2}
\subsection{Proof of Lemma~\ref{lemma:projectiondifference} }
\begin{proof}[Proof of Lemma~\ref{lemma:projectiondifference}]
The first inequality directly follows by equation (44) in \citet{su2019learning}. To prove the second bound, we write $ \Vb_{r_k} = \hat\Vb_{r_k} \Ab + \hat\Vb_{r_k}^\bot \Bb$, where $\Ab \in \RR^{r_k\times r_k}$, $\Bb \in\RR^{ (n - r_k) \times r_k }$. 
Let $\xi_1 = C'(\lambda_{r_k} - \lambda_{r_k + 1})^{-1} \cdot \sqrt{ \log(1/\delta)/n }$, $\xi_2 =  C''' M^2 \sqrt{ \log(r_k / \delta) / n}$ be the bounds given in the first inequality and Lemma~\ref{lemma:projectionconcentration}. By the first inequality, we have with high probability
\begin{align*}
    \| \Bb \|_F = \| \Bb^\top \|_F =\| \Vb_{{r_k}}^\top \hat\Vb_{{r_k}}^\bot \|_F \leq \xi_1.
\end{align*}
Moreover, since $\Vb_{r_k}^\top \Vb_{r_k} = \Ab^\top \Ab + \Bb^\top \Bb$, by Lemma~\ref{lemma:projectionconcentration} we have
\begin{align*}
    \| \Ab \Ab^\top - \Ib \|_2 = \| \Ab^\top \Ab  - \Ib \|_2 \leq \| \Vb_{r_k}^\top \Vb_{r_k} - \Ib \|_2 + \| \Bb^\top \Bb \|_2 \leq r_k \xi_2 + \xi_1^2.
\end{align*}
Therefore
\begin{align*}
    &\| \Vb_{{r_k}} \Vb_{{r_k}}^\top -  \hat\Vb_{{r_k}} \hat\Vb_{{r_k}}^\top \|_2 \\
    &\qquad= \| \hat\Vb_{{r_k}} \Ab \Ab^\top \hat\Vb_{{r_k}}^\top + \hat\Vb_{{r_k}} \Ab \Bb^\top (\hat\Vb_{{r_k}}^\bot)^\top + \hat\Vb_{{r_k}}^\bot \Bb \Ab^\top \hat\Vb_{{r_k}}^\top  + \hat\Vb_{{r_k}}^\bot \Bb \Bb^\top (\hat\Vb_{{r_k}}^\bot)^\top  -  \hat\Vb_{{r_k}} \hat\Vb_{{r_k}}^\top \|_2\\
    &\qquad\leq \| \hat\Vb_{{r_k}} (\Ab \Ab^\top - \Ib ) \hat\Vb_{{r_k}}^\top \|_2 + \cO(\| \Bb \|_2) \\
    &\qquad=  \| \Ab \Ab^\top - \Ib \|_2 + \cO(\| \Bb \|_2)\\
    &\qquad\leq \cO( r_k \xi_2 + \xi_1)
\end{align*}
Plugging in the definition of $\xi_1$ and $\xi_2$ completes the proof.
\end{proof}

\subsection{Proof of Lemma~\ref{lemma:explicitboundinsideball}}
The following lemma is a direct application of Proposition 1 in \cite{smale2009geometry} or  Proposition~10 in \citet{rosasco2010learning}. It bounds the difference between the eigenvalues of NTK and their finite-width counterparts.
\begin{lemma}\label{lemma:eigenvalueconcentration}
For any $\delta > 0$, with probability at least $1 - \delta$, 
$ | \lambda_i - \hat\lambda_i | \leq \cO(\sqrt{\log(1/\delta) / n })$ for $i \in [n]$.
\end{lemma}

The following lemma gives a recursive formula with is key to the proof of Lemma~\ref{lemma:explicitboundinsideball}.
\begin{lemma}\label{lemma:recursiveformula}
Suppose that the iterates of gradient descent $\Wb^{(0)},\ldots, \Wb^{(t)}$ are inside the ball $\cB(\Wb^{(0)}, \omega )$. If $\omega\leq \cO([\log(m)]^{-3/2})$, then with probability at least $1 - \cO(n^2) \cdot \exp[-\Omega(m\omega^{2/3} )] $,
\begin{align*}
    \yb - \hat\yb^{({t'}+1)} = [\Ib - (\eta m\theta^2 / n) \Kb^{\infty} ] (\yb - \hat\yb^{({t'})}) + \eb^{(t)}, ~\| \eb^{({t'})} \|_2 \leq \tilde\cO( \omega^{1/3}\eta m \theta^2)\cdot \| \yb - \hat\yb^{({t'})} \|_2 
\end{align*}
for all ${t'}=0,\ldots, t-1$, where $\yb = (y_1,\ldots, y_n)^\top$, $\hat\yb^{({t'})} = \theta\cdot (f_{\Wb^{({t'})}}(\xb_1),\ldots, f_{\Wb^{({t'})}}(\xb_n))^\top$.  
\end{lemma}

We also have the following lemma, which provides a uniform bound of the neural network function value over $\cB(\Wb^{(0)}, \omega)$. 
\begin{lemma}\label{lemma:functionvaluebound_uniformW}
Suppose that $m \geq \Omega( \omega^{-2/3} \log(n/\delta))$ and $\omega \leq \cO([\log(m)]^{-3})$. Then with probability at least $ 1 - \delta $, $ |f_{\Wb}(\xb_i)| \leq \cO( \sqrt{\log(n/\delta)} + \omega \sqrt{m}) $
for all $\Wb \in \cB( \Wb^{(0)},\omega )$ $i\in [n]$.
\end{lemma}

\begin{proof}[Proof of Lemma~\ref{lemma:explicitboundinsideball}]
Denote $\ub^{(t)} = \yb - \hat\yb^{(t)}$, $t\in {T}$. 
Then we have
\begin{align*}
    \| (\hat\Vb_{r_k}^{\bot})^\top \ub^{({t'}+1)} \|_2 &\leq \| (\hat\Vb_{r_k}^{\bot})^\top [\Ib - (\eta m \theta^2/ n) \Kb^{\infty}] \ub^{({t'})} \|_2 + \tilde\cO( \omega^{1/3}\eta m \theta^2)\cdot \| \ub^{({t'})} \|_2\\
    &\leq \| (\hat\Vb_{r_k}^{\bot})^\top \ub^{({t'})} \|_2 + \tilde\cO( \omega^{1/3}\eta m \theta^2)\cdot \sqrt{n}\cdot \tilde\cO(1 + \omega \sqrt{m}),
\end{align*}
where the first inequality follows by Lemma~\ref{lemma:recursiveformula}, and the second inequality follows by Lemma~\ref{lemma:functionvaluebound_uniformW}. Therefore we have
\begin{align*}
    \| (\hat\Vb_{r_k}^{\bot})^\top \ub^{({t'})} \|_2 \leq \| (\hat\Vb_{r_k}^{\bot})^\top \ub^{(0)} \|_2 +  {t'} \cdot \omega^{1/3}\eta m \theta^2 \cdot \sqrt{n} \cdot
    \tilde\cO(1 + \omega \sqrt{m}),
\end{align*}
for ${t'} = 0,\ldots, t$. This completes the proof of \eqref{eq:lemma_explicitboundinsideball_eq1}. Similarly, we have
\begin{align*}
    \| \hat\Vb_{r_k}^\top \ub^{({t'}+1)} \|_2 &\leq \| \hat\Vb_{r_k}^\top [\Ib - (\eta m \theta^2/ n) \Kb^{\infty}] \ub^{({t'})} \|_2 + \tilde\cO( \omega^{1/3}\eta m \theta^2)\cdot \| \ub^{({t'})} \|_2\\
    &\leq ( 1 - \eta m \theta^2\hat\lambda_{r_k}) \| \hat\Vb_{r_k}^\top \ub^{({t'})} \|_2 + \tilde\cO( \omega^{1/3}\eta m \theta^2 )\cdot ( \| \hat\Vb_{r_k}^\top \ub^{({t'})} \|_2 + \| (\hat\Vb_{r_k}^\bot)^\top \ub^{({t'})} \|_2)\\
    &\leq ( 1 - \eta m \theta^2 \lambda_{r_k}/2) \| \hat\Vb_{r_k}^\top \ub^{({t'})} \|_2 + \tilde\cO( \omega^{1/3}\eta m \theta^2 )\cdot  \| (\hat\Vb_{r_k}^\bot)^\top \ub^{({t'})} \|_2\\
    &\leq ( 1 - \eta m \theta^2 \lambda_{r_k}/2) \| \hat\Vb_{r_k}^\top \ub^{({t'})} \|_2 + {t'} \cdot (\omega^{1/3}\eta m \theta^2)^2 \cdot\sqrt{n} \cdot \tilde\cO(1 + \omega \sqrt{m}) \\
    &\quad + \tilde\cO( \omega^{1/3}\eta m \theta^2 )\cdot  \| (\hat\Vb_{r_k}^\bot)^\top \ub^{(0)} \|_2
\end{align*}
for ${t'} = 0,\ldots,t-1$, where the third inequality is by Lemma~\ref{lemma:eigenvalueconcentration} and the assumption that $\omega \leq  \tilde\cO(\lambda_{r_k}^{3})$, $ n \geq \tilde\cO(\lambda_{r_k}^{-2})$, and the fourth inequality is by \eqref{eq:lemma_explicitboundinsideball_eq1}. Therefore we have
\begin{align*}
    \| \hat\Vb_{r_k}^\top \ub^{({t'})} \|_2 &\leq ( 1 - \eta m \theta^2 \lambda_{r_k}/2)^{t'} \| \hat\Vb_{r_k}^\top \ub^{(0)} \|_2 + t' \cdot (\eta m \theta^2 \lambda_{r_k}/2)^{-1} \cdot (\omega^{1/3}\eta m \theta^2)^2 \cdot \sqrt{n}\cdot \tilde\cO(1 + \omega \sqrt{m}) \\
    &\quad + (\eta m \theta^2 \lambda_{r_k}/2)^{-1} \cdot \tilde\cO( \omega^{1/3}\eta m \theta^2 )\cdot  \| (\hat\Vb_{r_k}^\bot)^\top \ub^{(0)} \|_2\\
    &= ( 1 - \eta m \theta^2 \lambda_{r_k}/2)^{t'} \| \hat\Vb_{r_k}^\top \ub^{(0)} \|_2 + {t'} \lambda_{r_k}^{-1} \cdot \omega^{2/3} \eta m \theta^2 \cdot\sqrt{n}\cdot \tilde\cO(1 + \omega \sqrt{m}).
\end{align*}
This completes the proof of \eqref{eq:lemma_explicitboundinsideball_eq2}. Finally, for \eqref{eq:lemma_explicitboundinsideball_eq4}, by assumption we have $\omega^{1/3}\eta m \theta^2 \leq \tilde \cO(1)$. Therefore 
\begin{align*}
    \| \ub^{({t'}+1)} \|_2 &\leq \| [\Ib - (\eta m \theta^2 / n) \Kb^{\infty}] \hat\Vb_{r_k} \hat\Vb_{r_k}^\top \ub^{({t'})} \|_2 + \| [\Ib - (\eta m \theta^2 / n) \Kb^{\infty}] \hat\Vb_{r_k}^{\bot} (\hat\Vb_{r_k}^{\bot})^\top \ub^{({t'})} \|_2\\
    &\quad + \tilde\cO( \omega^{1/3}\eta m \theta^2 )\cdot \| \hat\Vb_{r_k}^\top \ub^{({t'})} \|_2 + \tilde\cO( \omega^{1/3}\eta m \theta^2 )\cdot \| (\hat\Vb_{r_k}^\bot)^\top \ub^{({t'})} \|_2 \\
    &\leq (1 - \eta m \theta^2 \hat\lambda_{r_k} ) \| \hat\Vb_{r_k}^\top \ub^{({t'})} \|_2 + \tilde\cO( \omega^{1/3}\eta m \theta^2 )\cdot \| \hat\Vb_{r_k}^\top \ub^{({t'})} \|_2 + \tilde\cO( 1 )\cdot \| (\hat\Vb_{r_k}^\bot)^\top \ub^{({t'})} \|_2\\
    &\leq (1 - \eta m \theta^2 \lambda_{r_k} / 2) \| \hat\Vb_{r_k}^\top \ub^{({t'})} \|_2 + \tilde\cO( 1)\cdot \| (\hat\Vb_{r_k}^\bot)^\top \ub^{({t'})} \|_2\\
    &\leq (1 - \eta m \theta^2 \lambda_{r_k} / 2) \| \hat\Vb_{r_k}^\top \ub^{({t'})} \|_2 + \tilde\cO( 1 )\cdot \| (\hat\Vb_{r_k}^{\bot})^\top \ub^{(0)} \|_2  + {t'} \omega^{1/3}\eta m \theta^2 \sqrt{n} \cdot \tilde\cO(1 + \omega \sqrt{m})
\end{align*}
for ${t'} = 0,\ldots,t-1$, where the third inequality is by Lemma~\ref{lemma:eigenvalueconcentration} and the assumption that $\omega \leq  \tilde\cO(\lambda_{r_k}^{3})$, and the fourth inequality follows by \eqref{eq:lemma_explicitboundinsideball_eq1}. 
Therefore we have
\begin{align*}
    \| \ub^{({t'})} \|_2 &\leq \cO(\sqrt{n}) \cdot (1 - \eta m \theta^2 \lambda_{r_k} / 2)^{t'}  + \tilde\cO ( (\eta m \theta^2 \lambda_{r_k})^{-1} )\cdot \| (\hat\Vb_{r_k}^\bot)^\top \ub^{(0)} \|_2 + \lambda_{r_k}^{-1} {t'}  \omega^{1/3}\sqrt{n} \cdot  \tilde\cO(1 + \omega \sqrt{m}).
\end{align*}
This finishes the proof. 
\end{proof}

\section{Appendix D: Proof of lemmas in Appendix C}\label{sec:proof_of_lemmas_3}

\subsection{Proof of Lemma~\ref{lemma:recursiveformula}}

\begin{lemma}[\citet{cao2019generalizationsgd}]\label{lemma:semilinear}
There exists an absolute constant $\kappa$ such that, with probability at least $1 - \cO(n) \cdot \exp[-\Omega(m\omega^{2/3} )] $ over the randomness of $\Wb^{(0)}$, for all $i\in [n]$ and $\Wb,\Wb'\in \cB(\Wb^{(0)},\omega)$ with $ \omega \leq \kappa [\log(m)]^{-3/2}$, it holds uniformly that 
\begin{align*}
    | f_{\Wb'}(\xb_i) - f_{\Wb}(\xb_i) - \la \nabla_\Wb f_{\Wb}(\xb_i) , \Wb' - \Wb \ra | \leq \cO\Big( \omega^{1/3}\sqrt{m\log(m)} \Big) \cdot \| \Wb_1' - \Wb_1 \|_2.
\end{align*}
\end{lemma}

\begin{lemma}\label{lemma:gradientdifference}
If $\omega \leq \cO([\log(m)]^{-3/2})$, 
then with probability at least $1 - \cO(n)\cdot \exp[-\Omega(m \omega^{2/3})]$, 
\begin{align*}
    &\| \nabla_{\Wb} f_{\Wb}(\xb_i) - \nabla_{\Wb} f_{\Wb^{(0)}}(\xb_i) \|_F \leq \cO( \omega^{1/3}\sqrt{m}),\\
    & | \la \nabla_{\Wb} f_{\Wb}(\xb_i), \nabla_{\Wb} f_{\Wb}(\xb_j) \ra - \la \nabla_{\Wb} f_{\Wb^{(0)}}(\xb_i), \nabla_{\Wb} f_{\Wb^{(0)}}(\xb_j) \ra | \leq \cO( \omega^{1/3}m)
\end{align*}
for all $\Wb\in \cB(\Wb^{(0)}, \omega)$ and $i\in [n]$.
\end{lemma}

\begin{proof}[Proof of Lemma~\ref{lemma:recursiveformula}]
The gradient descent update formula gives
\begin{align}\label{eq:recursiveformulaproof_gdrule}
    \Wb^{(t+1)} = \Wb^{(t)} + \frac{2\eta}{n} \sum_{i=1}^n (y_i - \theta f_{\Wb^{(t)}}(\xb_i) )\cdot \theta \nabla_{\Wb} f_{\Wb^{(t)}}(\xb_i).
\end{align}
For any $j\in[n]$, subtracting $\Wb^{(t)}$ and applying inner product with $\theta \nabla_{\Wb} f_{\Wb^{(t)}}(\xb_j)$ gives
\begin{align*}
    \theta \la \nabla_{\Wb} f_{\Wb^{(t)}}(\xb_j), \Wb^{(t+1)} - \Wb^{(t)} \ra &= \frac{2\eta \theta^2}{n} \sum_{i=1}^n (y_i - \hat y_i^{(t)} )\cdot \la \nabla_{\Wb} f_{\Wb^{(t)}}(\xb_j), \nabla_{\Wb} f_{\Wb^{(t)}}(\xb_i) \ra. 
\end{align*}
Further rearranging terms then gives
\begin{align}
    y_j - (\hat\yb^{(t+1)})_j &= y_j - (\hat\yb^{(t)})_j -\frac{2\eta m \theta^2}{n} \sum_{i=1}^n (y_i - \theta f_{\Wb^{(t)}}(\xb_i) )\cdot \Kb^{\infty}_{i,j} + I_{1,j,t} + I_{2,j,t} + I_{3,j,t}, \label{eq:recursiveformulaproof_eq1}
\end{align}
where 
\begin{align*}
    &I_{1,j,t} = -\frac{2\eta \theta^2}{n} \sum_{i=1}^n (y_i - \theta f_{\Wb^{(t)}}(\xb_i) )\cdot [\la \nabla_{\Wb} f_{\Wb^{(t)}}(\xb_j), \nabla_{\Wb} f_{\Wb^{(t)}}(\xb_i) \ra - m \Kb_{i,j}^{(0)} ],\\
    &I_{2,j,t} = -\frac{2\eta m \theta^2}{n} \sum_{i=1}^n (y_i - \theta f_{\Wb^{(t)}}(\xb_i) )\cdot (\Kb_{i,j}^{(0)} - \Kb_{i,j}^{\infty} ),\\
    &I_{3,j,t} = -\theta \cdot[f_{\Wb^{(t+1)}}(\xb_j) - f_{\Wb^{(t)}}(\xb_j) - \la \nabla_{\Wb} f_{\Wb^{(t)}}(\xb_j) , \Wb^{(t+1)} - \Wb^{(t)} \ra].
\end{align*}
For $I_{1,j,t}$, by Lemma \ref{lemma:gradientdifference}, we have
\begin{align*}
    |I_{1,j,t}| 
    \leq 
    \cO ( \omega^{1/3} \eta m \theta^2) \cdot \frac{1}{n} \sum_{i=1}^n |y_i - \theta f_{\Wb^{(t)}}(\xb_i) | \leq 
    \cO ( \omega^{1/3} \eta m \theta^2) \cdot \| \yb - \hat\yb^{(t)} \|_2 / \sqrt{n}
\end{align*}
with probability at least $1 - \cO(n)\cdot \exp[-\Omega(m \omega^{2/3})]$. For $I_{2,j,t}$, by Bernstein inequality and union bound, with probability at least $1 - \cO(n^2) \cdot \exp( -\Omega(m \omega^{2/3})) $, we have
\begin{align*}
    \big| \Kb^\infty_{i,j} - \Kb^{(0)}_{i,j} \big| \leq \cO (\omega^{1/3})
\end{align*}
for all $i,j\in [n]$. Therefore
\begin{align*}
    |I_{2,j,t}| \leq \cO( \omega^{1/3} \eta m \theta^2 ) \cdot \frac{1}{n} \sum_{i=1}^n |y_i - \theta f_{\Wb^{(t)}}(\xb_i) | \leq \cO( \omega^{1/3} \eta m \theta^2) \cdot \| \yb - \hat\yb^{(t)} \|_2 / \sqrt{n}.
\end{align*}
For $I_{3,j,t}$, we have
\begin{align*}
    I_{3,j,t} &\leq \tilde\cO(\omega^{1/3} \sqrt{m} \theta) \cdot  \| \Wb_1^{(t+1)} - \Wb_1^{(t)} \|_2\\
    & \leq \tilde\cO(\omega^{1/3} \sqrt{m} \theta) \cdot \frac{2\eta}{n} \sum_{i=1}^n | y_i - \theta f_{\Wb^{(t)}}(\xb_i) | \cdot \theta \cdot \| \nabla_{\Wb_1} f_{\Wb^{(t)}}(\xb_i) \|_2\\
    &\leq \tilde\cO(\omega^{1/3} \eta m \theta^2) \cdot \frac{1}{n} \sum_{i=1}^n | y_i - \theta f_{\Wb^{(t)}}(\xb_i) |\\
    &\leq \tilde\cO(\omega^{1/3} \eta m \theta^2) \cdot \| \yb - \hat\yb^{(t)} \|_2 / \sqrt{n},
\end{align*}
where the first inequality follows by Lemmas~\ref{lemma:semilinear}, the second inequality is obtained from \eqref{eq:recursiveformulaproof_gdrule}, and the third inequality follows by Lemma \ref{lemma:NNgradient_uppbound}. 
Setting the $j$-th entry of $\eb^{(t)}$ as $I_{1,j,t}+I_{2,j,t}+ I_{3,j,t}$ and writing \eqref{eq:recursiveformulaproof_eq1} into matrix form completes the proof.
\end{proof}

\subsection{Proof of Lemma~\ref{lemma:functionvaluebound_uniformW}}
\begin{proof}[Proof of Lemma~\ref{lemma:functionvaluebound_uniformW}]
By Lemmas~\ref{lemma:semilinear} and \ref{lemma:NNgradient_uppbound}, we have
\begin{align*}
    |f_{\Wb}(\xb_i) - f_{\Wb^{(0)}}(\xb_i)| 
    &\leq \| \nabla_{\Wb_1} f_{\Wb^{(0)}}(\xb_i) \|_F \| \Wb_1 - \Wb_1^{(0)} \|_F + \| \nabla_{\Wb_2} f_{\Wb^{(0)}}(\xb_i) \|_F \| \Wb_2 - \Wb_2^{(0)} \|_F \\
    &\quad + \cO ( \omega^{1/3}\sqrt{m \log(m) } ) \cdot \| \Wb_1 - \Wb_1^{(0)} \|_2\\
    &\leq \cO(\omega \sqrt{m}),
\end{align*}
where the last inequality is by the assumption $\omega \leq [\log(m)]^{-3}$. 
Applying triangle inequality and  Lemma~\ref{lemma:initialfunctionvaluebound} then gives 
\begin{align*}
    |f_{\Wb}(\xb_i)| &\leq |f_{\Wb^{(0)}}(\xb_i)| + |f_{\Wb}(\xb_i) - f_{\Wb^{(0)}}(\xb_i)| \leq \cO(\sqrt{\log(n/\delta)}) + 
    \cO(\omega \sqrt{m})= \cO( \sqrt{\log(n/\delta)} + \omega \sqrt{m}),
\end{align*}
This completes the proof.
\end{proof}

\section{Appendix E: Proof of lemmas in Appendix D}
\subsection{Proof of Lemma~\ref{lemma:gradientdifference}}
Denote
\begin{align*}
    &\Db_{i} = \diag\big( \ind\{(\Wb_{1} \xb_{i} )_1 > 0 \},\ldots, \ind\{(\Wb_{1} \xb_{i} )_m > 0 \} \big),\\
    &\Db_{i}^{(0)} = \diag\big( \ind\{(\Wb_{1}^{(0)} \xb_{i} )_1 > 0 \},\ldots, \ind\{(\Wb_{1}^{(0)} \xb_{i} )_m > 0 \} \big).
\end{align*}
The following lemma is given in Lemma~8.2 in \citet{allen2018convergence}.
\begin{lemma}[\citet{allen2018convergence}]\label{lemma:differencesparsity}
If $\omega \leq \cO( [\log(m)]^{-3/2})$, 
then with probability at least $1 - \cO(n)\cdot \exp[-\Omega(m \omega^{2/3})]$, $\| \Db_{i} - \Db_{i}^{(0)} \|_0 \leq \cO(\omega^{2/3} m) $ for all $\Wb \in \cB( \Wb^{(0)},\omega )$, $i\in [n]$.
\end{lemma}

\begin{proof}[Proof of Lemma~\ref{lemma:gradientdifference}]
By direct calculation, we have
\begin{align*}
    \nabla_{\Wb_1} f_{\Wb^{(0)}}(\xb_i) = \sqrt{m}\cdot \Db_i^{(0)} \Wb_2^{(0)\top} \xb_i^\top, \nabla_{\Wb_1} f_{\Wb}(\xb_i) = \sqrt{m}\cdot \Db_i \Wb_2^\top \xb_i^\top.
\end{align*}
Therefore we have
\begin{align*}
    \| \nabla_{\Wb_1} f_{\Wb}(\xb_i) - \nabla_{\Wb_1} f_{\Wb^{(0)}}(\xb_i)\|_F &= \sqrt{m}\cdot \| \Db_i \Wb_2^\top \xb_i^\top - \Db_i^{(0)} \Wb_2^{(0)\top} \xb_i^\top \|_F\\
    &= \sqrt{m}\cdot \| \xb_i \Wb_2\Db_i - \xb_i \Wb_2^{(0)}\Db_i^{(0)} \|_F\\
    &= \sqrt{m}\cdot \| \Wb_2\Db_i - \Wb_2^{(0)}\Db_i^{(0)} \|_2\\
    &\leq \sqrt{m}\cdot \| \Wb_2^{(0)}(\Db_i^{(0)} - \Db_i)\|_2 + \sqrt{m}\cdot \| (\Wb_2^{(0)} - \Wb_2 )\Db_i\|_2
\end{align*}
By Lemma~7.4 in \citet{allen2018convergence} and Lemma~\ref{lemma:differencesparsity}, with probability at least $1 - n \cdot \exp[ - \Omega(m) ] $, $\sqrt{m}\cdot \| \Wb_2^{(0)}(\Db_i^{(0)} - \Db_i)\|_F \leq \cO(\omega^{1/3} \sqrt{m})$ for all $i\in [n]$. Moreover, clearly $\| (\Wb_2^{(0)} - \Wb_2 )\Db_i\|_2 \leq \| \Wb_2^{(0)} - \Wb_2 \|_2 \leq \omega$. Therefore
\begin{align*}
    \| \nabla_{\Wb_1} f_{\Wb}(\xb_i) - \nabla_{\Wb_1} f_{\Wb^{(0)}}(\xb_i)\|_F \leq \cO(\omega^{1/3} \sqrt{m})
\end{align*}
for all $i\in[n]$. This proves the bound for the first layer gradients. For the second layer gradients, we have
\begin{align*}
    \nabla_{\Wb_2} f_{\Wb^{(0)}}(\xb_i) = \sqrt{m}\cdot [\sigma(\Wb_1^{(0)} \xb_i)]^{\top} , \nabla_{\Wb_2} f_{\Wb}(\xb_i) = \sqrt{m}\cdot [\sigma(\Wb_1 \xb_i)]^{\top}.
\end{align*}
It therefore follows by the $1$-Lipschitz continuity of $\sigma(\cdot)$ that
\begin{align*}
    \| \nabla_{\Wb_2} f_{\Wb}(\xb_i) - \nabla_{\Wb_2} f_{\Wb^{(0)}}(\xb_i) \|_2 \leq \sqrt{m}\cdot \| \Wb_1 \xb_i - \Wb_1^{(0)} \xb_i \|_2 \leq \omega\sqrt{m} \leq \omega^{1/3}\sqrt{m}.
\end{align*}
This completes the proof of the first inequality. The second inequality directly follows by triangle inequality and Lemma~\ref{lemma:NNgradient_uppbound}:
\begin{align*}
    | \la \nabla_{\Wb} f_{\Wb}(\xb_i), \nabla_{\Wb} f_{\Wb}(\xb_j) \ra - m\Kb^{(0)} | 
    &\leq | \la \nabla_{\Wb} f_{\Wb}(\xb_i) - \nabla_{\Wb} f_{\Wb^{(0)}}(\xb_i), \nabla_{\Wb} f_{\Wb}(\xb_j) \ra |\\
    &\quad+ | \la \nabla_{\Wb} f_{\Wb^{(0)}}(\xb_i), \nabla_{\Wb} f_{\Wb}(\xb_j) - \nabla_{\Wb} f_{\Wb^{(0)}}(\xb_j) \ra |\\
    &\leq \cO( \omega^{1/3}m).
\end{align*}
This finishes the proof.
\end{proof}

\section{Appendix F: Additional Experimental Results}

As is discussed in the main paper, when using freshly sampled points, we are actually estimating the projection length of residual function $r(\xb) = f^*(\xb) - \theta f_{\Wb^{(t)}}(\xb)$ onto the given Gegenbauer polynomial $P_k(\xb)$. Here we present in Figure \ref{fig3} a comparison between the projection length using training data and that using test data. An interesting phenomenon is that the network generalizes well on the lower-order Gegenbauer polynomial and along the highest-order Gegenbauer polynomial the network suffers overfitting.

\begin{figure}[H]
     \centering
     \subfigure[projection onto training data]{\includegraphics[width=0.47\textwidth]{figures/linear_train_111.pdf}}
     \subfigure[projection onto test data]{\includegraphics[width=0.47\textwidth]{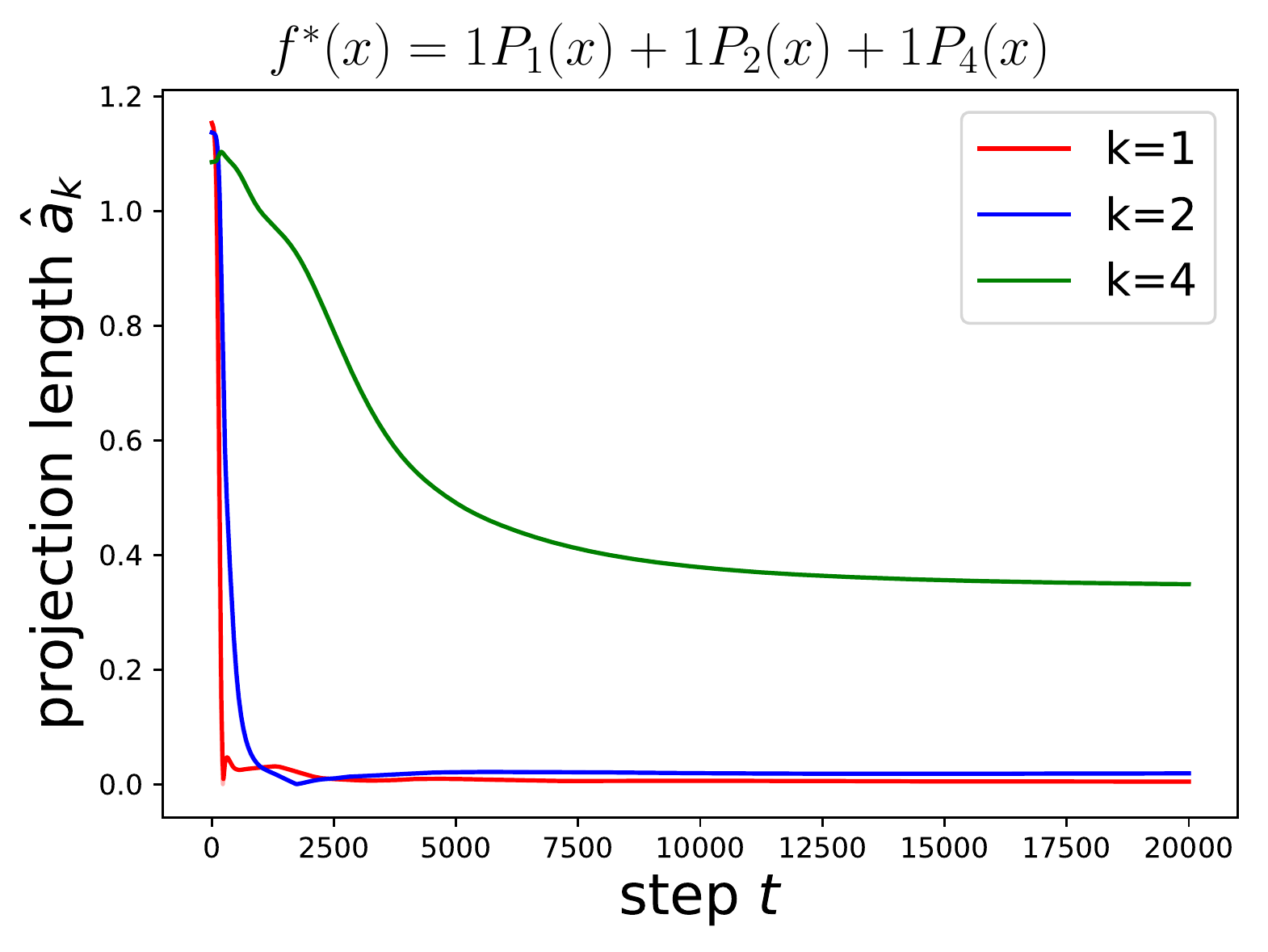}}
     \subfigure[projection onto training data]{\includegraphics[width=0.47\textwidth]{figures/linear_train_135.pdf}}
     \subfigure[projection onto test data]{\includegraphics[width=0.47\textwidth]{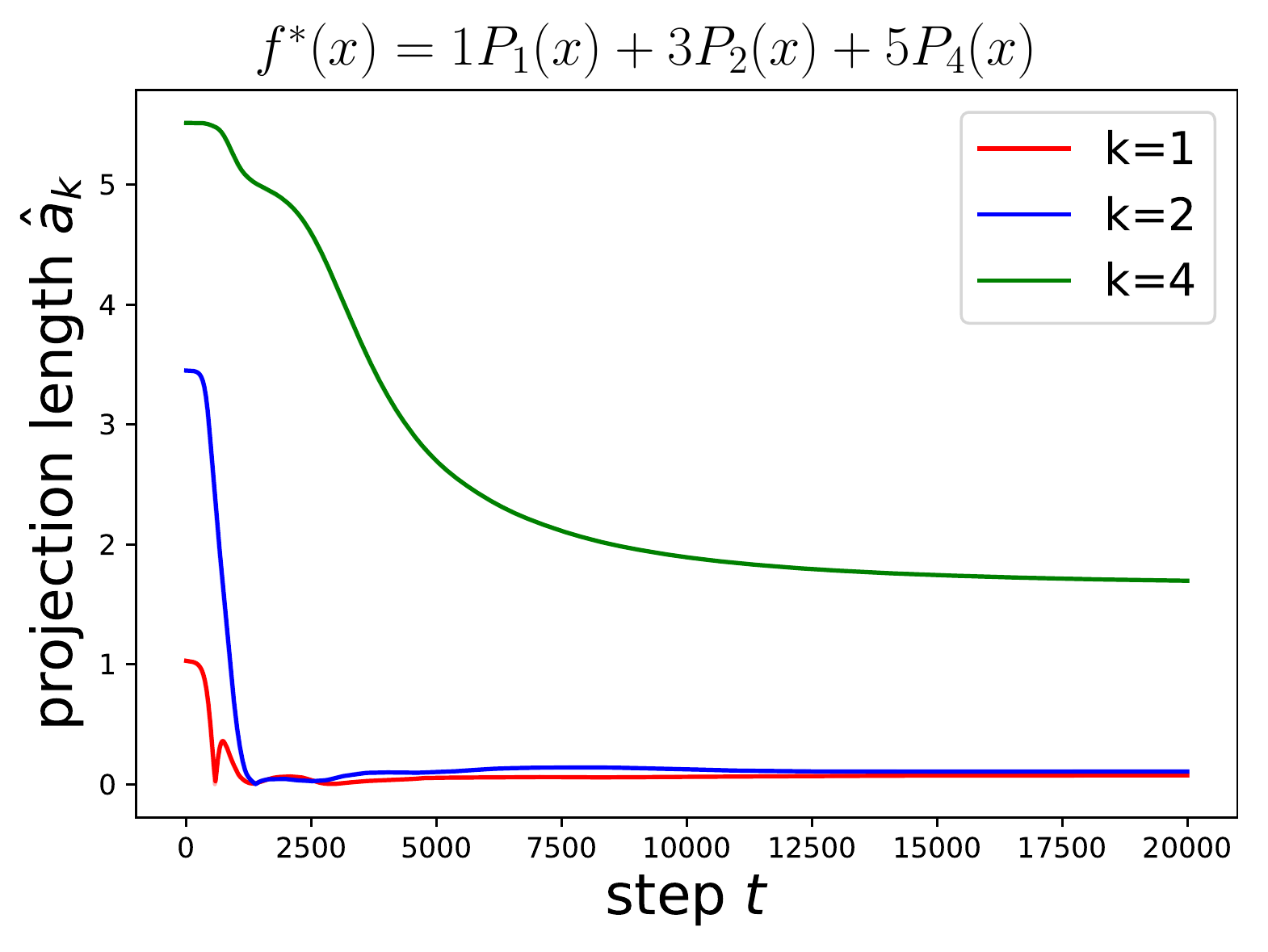}}
    \caption{Convergence curve for projection length onto vectors (determined by training data) and functions (estimated by test data). We can see that for low-order Gegenbauer polynomials, the network learns the function while for the high-order Gegenbauer polynomial, the network overfits the training data.}
    \label{fig3}
\end{figure}

\bibliographystyle{ims.bst}
\bibliography{ReLU.bib}

\end{document}